\documentclass[twoside,11pt]{article}

\usepackage{jmlr2e_}
\usepackage{amsmath}
\usepackage{amsfonts}
\usepackage{amssymb}
\usepackage{algorithm,algorithmic}
\usepackage{caption}
\usepackage{subfig}
\usepackage{multirow}
\usepackage{color}
\usepackage{pslatex}
\usepackage{textcomp}

\def\x{{\mathbf x}}
\def\z{{\mathbf z}}
\def\1{{\mathbf 1}}

\def\v{{\mathbf v}}
\def\0{{\mathbf 0}}
\def\X{{\mathbf X}}
\def\Prox{\text{Prox}}

\def\weights{\omega}
\def\gi{ {{\scriptscriptstyle \hspace*{-0.002cm}\mid\hspace*{-0.008cm}}  g}}
\def\hi{  {{\scriptscriptstyle \hspace*{-0.002cm}\mid\hspace*{-0.008cm}}  h}}

\def\alphab{{\boldsymbol\alpha}}

\def\y{{\mathbf y}}
\def\w{{\mathbf w}}
\def\D{{\mathbf D}}
\def\DD{{\mathcal D}}

\def\Sb{{\mathbf S}}

\def\d{{\mathbf d}}

\def\N{{\mathcal N}}

\def\s{{\mathbf s}}

\def\d{{\mathbf d}}

\def\u{{\mathbf u}}

\def\tildeD{{\bf \tilde D}}

\def\tildex{{\bf \tilde x}}

\def\Real{{\mathbb R}}

\def\u{{\mathbf u}}
\def\A{{\mathbf A}}

\def\TT{{\mathcal T}}
\def\I{{\mathbf I}}

\def\argmin{\operatornamewithlimits{arg\,min}}

\def\sign{\operatorname{sign}}
\def\st{~~\text{s.t.}~~}
\def\defin{\stackrel{\vartriangle}{=}}

 \newcommand{\RETURN}{ \STATE {\textbf{Return}} }

\newcommand{\R}[1]{\mathbb{R}^{#1}}
\newcommand{\RR}[2]{\mathbb{R}^{#1 \times #2}}
\newcommand{\G}{\mathcal{G}}

\newcommand{\InSet}[1]{\{1,\ldots,#1\}}
\newcommand{\Norm}[1]{\|#1\|}
\newcommand{\DualNorm}[1]{\|#1\|_{\ast}}

\newcommand{\NormDeux}[1]{\|#1\|_2}
\newcommand{\NormInf}[1]{\|#1\|_{\infty}}
\newcommand{\NormFro}[1]{\|#1\|_{\text{F}}}

\def \xib{{\boldsymbol\xi}}
\def \taub{{\boldsymbol\tau}}
\def \rhob{{\boldsymbol\rho}}
\def \kappab {{\boldsymbol\kappa}}

\newcommand{\refLemma}[1]{Lemma~\ref{#1}}

\newcommand{\refSec}[1]{Section~\ref{#1}}

\long\def\symbolfootnote[#1]#2{\begingroup\def\thefootnote{\fnsymbol{footnote}}\footnote[#1]{#2}\endgroup} 

\newcommand{\UpperSpace}{\vspace*{0.25cm}}
\newcommand{\LowerSpace}{\vspace*{0.25cm}}

\renewcommand\UpperSpace{ }
\renewcommand\LowerSpace{ }
\renewcommand{\cite}{\citep}

\jmlrheading{12}{2011}{-}{09/10; Revised 03/11}{-}{Rodolphe Jenatton, Julien Mairal, Guillaume Obozinski and Francis Bach}

\ShortHeadings{Proximal Methods for Hierarchical Sparse Coding}{Jenatton, Mairal, Obozinski and Bach}
\firstpageno{1}

\begin{document}

\title{Proximal Methods for Hierarchical Sparse Coding}

\author{\name Rodolphe Jenatton\thanks{Equal contribution.}~\hspace*{0.08cm}\thanks{Rodolphe Jenatton, Guillaume Obozinski, and Francis Bach are now affiliated to INRIA - Sierra Project-Team. Julien Mairal is now with the Statistics Department of the University of California at Berkeley. When this work was performed all authors were affiliated to INRIA - Willow Project-Team.} 
\email rodolphe.jenatton@inria.fr \\
        \name Julien Mairal$^*$$^\dagger$ \email julien.mairal@inria.fr \\
        \name Guillaume Obozinski$^\dagger$ \email guillaume.obozinski@inria.fr \\
        \name Francis Bach$^\dagger$ \email francis.bach@inria.fr \\
        \addr INRIA - WILLOW Project-Team\\
         Laboratoire d'Informatique de l'Ecole Normale Sup\'erieure (INRIA/ENS/CNRS UMR 8548)\\
         23, avenue d'Italie 75214 Paris CEDEX 13, France.
       }

\editor{}

\maketitle

\begin{abstract}
Sparse coding consists in representing signals as sparse linear combinations of
atoms selected from a dictionary.  We consider an extension of this framework
where the atoms are further assumed to be embedded in a tree. This is achieved
using a recently introduced tree-structured sparse regularization norm, which
has proven useful in several applications. This norm leads to regularized
problems that are difficult to optimize, and in this paper, we propose efficient
algorithms for solving them.  More precisely, we show that the proximal
operator associated with this norm is computable exactly via a dual approach
that can be viewed as the composition of elementary proximal operators.  Our
procedure has a complexity linear, or close to linear, in the number of atoms,
and allows the use of accelerated gradient techniques to solve the
tree-structured sparse approximation problem at the same computational cost as
traditional ones using the $\ell_1$-norm.  Our method is efficient and scales
gracefully to millions of variables, which we illustrate in two types of
applications: first, we consider \textit{fixed} hierarchical dictionaries of
wavelets to denoise natural images.  Then, we apply our optimization tools in
the context of \textit{dictionary learning}, where learned dictionary elements
naturally self-organize in a prespecified arborescent structure, leading to better
performance in reconstruction of natural image patches.  When applied to text
documents, our method learns hierarchies of topics, thus providing a
competitive alternative to probabilistic topic models.
\end{abstract}

\section{Introduction}
Modeling signals as sparse linear combinations of atoms selected from a
dictionary has become a popular paradigm in many fields, including signal
processing, statistics, and machine learning.  This line of research, also known
as \textit{sparse coding}, has witnessed the development of several well-founded
theoretical frameworks~\cite{Tibshirani1996,  Chen1998, Mallat1999, Tropp2004,
Tropp2006, Wainwright2009, Bickel2009} and the emergence of many efficient
algorithmic tools~\cite{Efron2004, Nesterov2007, Beck2009, Wright2009, Needell2009, Yuan2010a}.

In many applied settings, the structure of the problem at hand, such as, e.g., the
spatial arrangement of the pixels in an image, or the presence of variables corresponding to several levels of a given factor, induces relationships between dictionary elements. It is appealing to use this a priori knowledge about the
problem \textit{directly} to constrain the possible sparsity patterns.
For instance, when the dictionary elements are partitioned into predefined
groups corresponding to different types of features, one can
enforce a similar block structure in the sparsity pattern---that is, allow only that either
all elements of a group are part of the signal decomposition or that all are dismissed
simultaneously \citep[see][]{Yuan2006,Stojnic2009}.

This example can be viewed as a particular instance of \textit{structured
sparsity},  which has been lately the focus of a large amount of
research~\cite{Baraniuk2008, Zhao2009, Huang2009, Jacob2009, Jenatton2009,Micchelli2010}.
In this paper, we concentrate on a specific form of structured sparsity, which we call \emph{hierarchical sparse coding}: the dictionary
elements are assumed to be embedded in a directed tree $\mathcal{T}$, and the
sparsity patterns are constrained to form a \textit{connected and rooted subtree} of
$\mathcal{T}$~\cite{Donoho1997,Baraniuk1999,Baraniuk2002,Baraniuk2008,Zhao2009,Huang2009}.
This setting extends more generally to a forest of directed trees.\footnote{A tree is defined as a connected graph that contains no cycle~\cite[see][]{Ahuja1993}.}

In fact, such a hierarchical structure arises in many applications.
Wavelet decompositions lend themselves
well to this tree organization because of their multiscale structure, and benefit from it for image compression and
denoising~\cite{Shapiro1993,Crouse1998,Baraniuk1999,Baraniuk2002,Baraniuk2008,He2009,
Zhao2009,Huang2009}.  In the same vein, edge filters of natural image patches
can be represented in an arborescent fashion~\cite{Zoran2009}.  Imposing these
sparsity patterns has further proven useful in the context of hierarchical
variable selection, e.g., when applied to kernel methods~\cite{Bach2008}, to
log-linear models for the selection of potential orders~\cite{Schmidt2010}, and 
to bioinformatics, to exploit the tree structure of gene networks for
multi-task regression~\cite{Kim2009}.  Hierarchies of latent variables,
typically used in neural networks and deep learning architectures
\citep[see][and references therein]{Bengio2009}
have also emerged as a natural
structure in several applications, notably to model text documents. In particular, in
the context of \emph{topic models} \cite{Blei2003}, a hierarchical model of latent variables based on Bayesian non-parametric methods has been proposed by \citet{Blei2010} to model hierarchies of topics.

To perform hierarchical sparse coding, our work builds upon the
approach of~\citet{Zhao2009}
who first introduced a
sparsity-inducing norm $\Omega$ leading to this type of tree-structured sparsity
pattern.  We tackle the resulting nonsmooth convex optimization problem with
proximal methods~\citep[e.g.,][]{Nesterov2007, Beck2009, Wright2009, Combettes2010} 
and we show in this paper that its key step, the computation of the \textit{proximal operator}, 
can be solved exactly with a complexity linear, 
or close to linear, in the number of dictionary elements---that is,
with the same complexity as for classical $\ell_1$-sparse decomposition
problems~\cite{Tibshirani1996,Chen1998}.
Concretely, given an $m$-dimensional signal $\x$ along with a dictionary 
$\D = [\d^1,\dots,\d^p] \in \RR{m}{p}$ composed of $p$ atoms,
the optimization problem at the core of our paper can be written as 
$$
\min_{\alphab \in \R{p}} \dfrac{1}{2} \|\x - \D\alphab\|^2_2 + \lambda \Omega(\alphab),\ \mathrm{with}\ \lambda \geq 0.
$$
In this formulation, the sparsity-inducing norm $\Omega$ encodes 
a hierarchical structure among the atoms of $\D$, where this structure is assumed to be known beforehand. 
The precise meaning of \textit{hierarchical structure} and the definition of $\Omega$ will be made more formal in the next sections.
A particular instance of this problem---known as the \textit{proximal problem}---is central to our analysis and concentrates on the case where the dictionary $\D$ is orthogonal.

In addition to a speed benchmark that evaluates the performance of our proposed
approach in comparison with other convex optimization techniques, two types of
applications and experiments are considered.  First, we consider
settings where the dictionary is fixed and given a priori, corresponding for
instance to a basis of wavelets for the denoising of natural images. 
Second, we show how one can take advantage of this hierarchical sparse
coding in the context of dictionary
learning~\cite{Olshausen1997,Aharon2006,Mairal2010}, where the dictionary is
learned to adapt to the predefined tree structure. This extension of
dictionary learning is notably shown to share interesting connections with
hierarchical probabilistic topic models.

To summarize, the contributions of this paper are threefold:
\begin{itemize}
\item We show that the proximal
   operator for a tree-structured sparse regularization can be computed
   exactly in a finite number of operations using a dual approach. 
   Our approach is equivalent to computing a particular sequence of
   elementary proximal operators, and has a
   complexity linear, or close to linear, in the number of variables.
Accelerated gradient methods~\citep[e.g.,][]{Nesterov2007,
Beck2009, Combettes2010} can then be applied to solve large-scale tree-structured sparse
decomposition problems at the same computational cost as traditional ones using the $\ell_1$-norm.
\item We propose to use this regularization scheme to learn dictionaries
embedded in a tree, which, to the best of our knowledge, has not been done
before in the context of structured sparsity.
\item Our method establishes a bridge between hierarchical dictionary learning and hierarchical topic models
\cite{Blei2010}, which builds upon the interpretation of topic models as
multinomial PCA \cite{Buntine2002}, and can learn similar hierarchies of
topics. This point is discussed in Sections~\ref{sec:exp_txt_documents} and \ref{sec:ccl}.
\end{itemize}
Note that this paper extends a shorter version published in the proceedings of the international conference of machine learning~\cite{Jenatton2010a}.

\subsection{Notation}
Vectors are denoted by bold lower case letters and matrices by upper case ones.
We define for $q \geq 1$ the \mbox{$\ell_q$-norm} of a vector~$\x$ in~$\Real^m$ as
$\|\x\|_q \defin (\sum_{i=1}^m |\x_i|^q)^{{1}/{q}}$, where~$\x_i$ denotes the
$i$-th coordinate of~$\x$, and $\|\x\|_\infty \defin \max_{i=1,\ldots,m} |\x_i|
= \lim_{q \to \infty} \|\x\|_q$.  We also define the $\ell_0$-pseudo-norm as
the number of nonzero elements in a vector:\footnote{Note that it would
be more proper to write $\|\x\|_0^0$ instead of $\|\x\|_0$ to be consistent with the traditional notation $\|\x\|_q$.
However, for the sake of simplicity, we will keep this
notation unchanged in the rest of the paper.}
$\|\x\|_0 \defin \#\{i \st \x_i
\neq 0  \} = \lim_{q \to 0^+}  (\sum_{i=1}^m |\x_i|^q)$.  We consider the
Frobenius norm of a matrix~$\X$ in~$\Real^{m \times n}$: $\NormFro{\X} \defin
(\sum_{i=1}^m \sum_{j=1}^n \X_{ij}^2)^{{1}/{2}}$, where $\X_{ij}$ denotes the entry of~$\X$ at row $i$ and column $j$. Finally, for a scalar $y$, we denote $(y)_+ \defin \max(y,0)$.

The rest of this paper is organized as follows: Section \ref{sec:problem_statement}
presents related work and the problem we consider. Section \ref{sec:optimization}
is devoted to the algorithm we propose,
and Section~\ref{sec:dictionary_learning} introduces
the dictionary learning framework and shows how it can be used with tree-structured norms.
Section~\ref{sec:experiment} presents several experiments demonstrating the
effectiveness of our approach and Section~\ref{sec:ccl} concludes the paper.
\section{Problem Statement and Related Work}\label{sec:problem_statement}

Let us consider an input signal of dimension $m$, typically an image described by its $m$ pixels,
which we represent by a vector $\x$ in $\R{m}$.
In traditional sparse coding, we seek to approximate this signal by a sparse linear combination of atoms,
or dictionary elements, represented here by the columns of a matrix
$\D \defin [\d^1,\dots,\d^p]$ in $\RR{m}{p}$.
This can equivalently be expressed as
$\x \approx \D\alphab$ for some sparse vector $\alphab$ in $\R{p}$, i.e,
such that the number of nonzero coefficients~$\|\alphab\|_0$ is small compared to~$p$.
The vector $\alphab$ is referred to as the code, or decomposition, of the signal~$\x$.
\begin{figure}[htb!]
   \centering
   \includegraphics[width=0.6\textwidth]{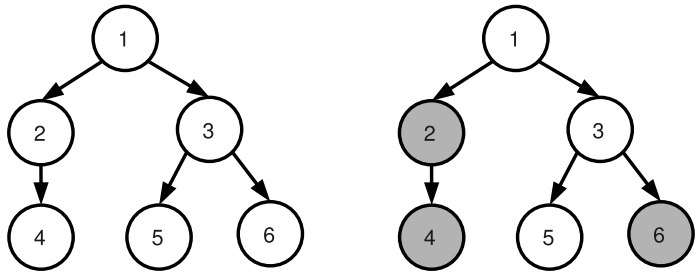}
   \caption{Example of a tree $\TT$ when $p=6$. With the rule we consider for the nonzero patterns, if we have $\alphab_5 \neq 0$,
   we must also have $\alphab_k \neq 0$ for $k$ in $\text{ancestors}(5)=\{1,3,5\}$.}\label{fig:plain_tree}
\end{figure}

In the rest of the paper, we focus on specific sets of nonzero coefficients---or simply, nonzero patterns---for the decomposition vector $\alphab$. In particular, we assume that we are given a tree\footnote{Our analysis straightforwardly extends to the case of a forest of trees; for simplicity, we consider a single tree $\TT$.}~$\TT$ whose $p$ nodes are indexed by $j$ in $\{1,\dots,p\}$.
We want the nonzero patterns of $\alphab$ to form a \textit{connected and rooted subtree} of $\TT$; in other words,
if $\text{ancestors}(j) \subseteq \{1,\dots,p\}$ denotes the set of indices corresponding to the ancestors\footnote{We consider that the set of ancestors of a node also contains the node itself.} of the node $j$ in $\TT$ (see Figure~\ref{fig:plain_tree}),
the vector $\alphab$ obeys the following rule
\begin{equation}\label{eq:ancestor_cond}
  \alphab_j \neq 0 \Rightarrow [\, \alphab_k \neq 0\ \text{for all}\ k\ \text{in ancestors}(j) \,].
\end{equation}
Informally,
we want to exploit the structure of $\TT$ in
the following sense: the decomposition of any signal $\x$ can involve
a dictionary element $\d^j$ \emph{only if the ancestors of $\d^j$ in the tree $\TT$
are themselves part of the decomposition}.

We now review previous work that has considered the sparse approximation problem with
tree-structured constraints~(\ref{eq:ancestor_cond}).  Similarly
to traditional sparse coding, there are basically two lines of research, that
either (A) deal with nonconvex and combinatorial formulations that are in
general computationally intractable and addressed with greedy algorithms, or
(B) concentrate on convex relaxations solved with convex programming methods.

\subsection{Nonconvex Approaches}
For a given sparsity level $s\geq 0$ (number of nonzero coefficients), the following nonconvex problem
\begin{equation}\label{eq:l0_formulation}
 \min_{\substack{\alphab\in\R{p}\\ \|\alphab\|_0 \leq s}} \frac{1}{2} \|\x-\D\alphab\|_2^2\quad \text{such that}\ \text{condition~(\ref{eq:ancestor_cond}) is respected},
\end{equation}
has been tackled by \citet{Baraniuk1999,Baraniuk2002} in the context of
wavelet approximations with a greedy procedure.  A penalized version of
problem~(\ref{eq:l0_formulation}) (that adds $\lambda\|\alphab\|_0$ to the
objective function in place of the constraint $\|\alphab\|_0 \leq s$) has been
considered by~\citet{Donoho1997}, while studying the more general problem of best approximation from dyadic partitions \citep[see Section 6 in][]{Donoho1997}.
Interestingly, the algorithm we introduce in
Section~\ref{sec:optimization} shares conceptual links with the
dynamic-programming approach of~\citet{Donoho1997}, which was also used
by~\citet{Baraniuk2008}, in the sense that the same order of traversal of the tree
is used in both procedures.  We investigate more thoroughly the relations
between our algorithm and this approach in Appendix~\ref{appendix:greedy}.

Problem~(\ref{eq:l0_formulation}) has been further studied for structured
compressive sensing~\cite{Baraniuk2008}, with a greedy algorithm that builds
upon \citet{Needell2009}.  Finally, \citet{Huang2009} have
proposed a formulation related to~(\ref{eq:l0_formulation}), with a nonconvex
penalty based on an infor\-ma\-tion-theoretic criterion.

\subsection{Convex Approach}
We now turn to a convex reformulation of the
constraint~(\ref{eq:ancestor_cond}), which is the starting point for the convex
optimization tools we develop in Section~\ref{sec:optimization}.

\subsubsection{Hierarchical Sparsity-Inducing Norms}

Condition~(\ref{eq:ancestor_cond}) can be equivalently expressed by its contrapositive,
thus leading to an intuitive way of penalizing the vector $\alphab$ to obtain tree-structured nonzero patterns.
More precisely, defining $\text{descendants}(j) \subseteq \{1,\dots,p\}$ analogously to $\text{ancestors}(j)$ for $j$ in $\{1,\dots,p\}$,
condition~(\ref{eq:ancestor_cond}) amounts to saying that \emph{if a dictionary element is not used in the decomposition, its descendants in the tree should not be used either}. 
Formally, this can be formulated as:
\begin{equation}\label{eq:descendant_cond}
  \alphab_j=0 \Rightarrow [\, \alphab_k=0\ \text{for all}\ k\ \text{in descendants}(j) \,].
\end{equation}
From now on, we denote by $\G$ the set defined by
$
\G \defin \{\text{descendants}(j); j\in\{1,\dots,p\}\},
$
and refer to each member~$g$ of $\G$ as a \textit{group} (Figure~\ref{fig:groups}).
To obtain a decomposition with the desired property~(\ref{eq:descendant_cond}),
one can naturally penalize the number of groups $g$ in $\G$ that are ``involved'' in the decomposition of~$\x$, i.e.,
that record at least one nonzero coefficient of $\alphab$:
\begin{equation}
\sum_{g\in\G} \delta^g,\ \text{with}\
\delta^g \defin
\begin{cases}
  1 & \text{if there exists}\ j\in g\ \text{such that}\ \alphab_j\neq 0,\\
  0 & \text{otherwise}.
\end{cases} \label{eq:nonconvex}
\end{equation}
While this intuitive penalization is nonconvex (and not even continuous), a convex proxy has been introduced by~\citet{Zhao2009}.
 It was further considered by \citet{Bach2008, Kim2009, Schmidt2010} in several different contexts.
For any vector $\alphab \in \R{p}$, let us define
\begin{displaymath} 
 \Omega(\alphab) \defin \sum_{g\in\G} \weights_g \|\alphab_{\gi}\|,
\end{displaymath}
where $\alphab_{\gi}$ is the vector of size $p$ whose coordinates are equal
to those of $\alphab$ for indices in the set $g$, and to $0$ otherwise\footnote{Note the difference with the notation $\alphab_g$, which is often used in the literature on structured sparsity, where $\alphab_g$ is a vector of size $|g|$.}.
The notation $\Norm{.}$ stands in practice either for the $\ell_2$- or $\ell_\infty$-norm,
and $(\weights_g)_{g \in \G}$ denotes some positive weights\footnote{For a complete definition of $\Omega$ for any $\ell_q$-norm,
a discussion of the choice of $q$, and a strategy for choosing the weights $\weights_g$ \citep[see][]{Zhao2009,Kim2009}.}.
As analyzed by~\citet{Zhao2009} and~\citet{Jenatton2009}, when penalizing by $\Omega$,
some of the vectors $\alphab_{\gi}$ are set to zero for some $g\in\G$.\footnote{It has been further shown by~\citet{Bach2010a} that the convex envelope of the nonconvex function of Eq.~(\ref{eq:nonconvex}) is in fact $\Omega$ with $\|.\|$ being the $\ell_\infty$-norm.}
Therefore, the components of~$\alphab$ corresponding to some complete subtrees of
$\mathcal{T}$ are set to zero, which exactly matches condition~(\ref{eq:descendant_cond}), as illustrated in Figure~\ref{fig:groups}.
\begin{figure}[hbtp!]
   \centering
   \includegraphics[width=0.6\textwidth]{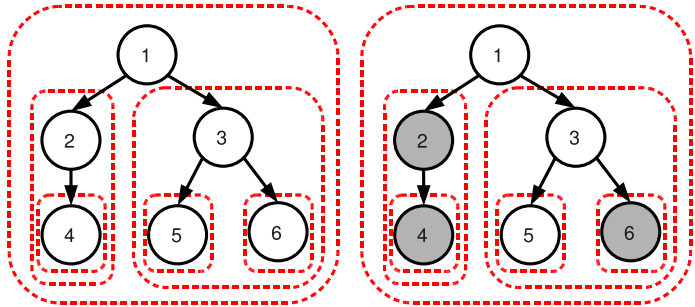}
   \caption{Left: example of a tree-structured set of groups $\G$ (dashed contours in red), corresponding to a tree $\mathcal{T}$ with $p=6$ nodes represented by black circles.
   Right: example of a sparsity pattern induced by the tree-structured norm corresponding to $\G$: the groups $\{2,4\},\{4\}$ and $\{6\}$ are set to zero, so that the corresponding nodes (in gray) that form subtrees of $\mathcal{T}$ are removed.
   The remaining nonzero variables $\{1,3,5\}$ form a rooted and connected subtree of $\mathcal{T}$.
   This sparsity pattern obeys the following equivalent rules: (i) if a node is selected, the same goes for all its ancestors. (ii) if a node is not selected, then its descendant are not selected.
}\label{fig:groups}
\end{figure}

Note that although we presented for simplicity this
hierarchical norm in the context of a single tree with a single element at
each node, it can easily be extended to the case of forests of trees, and/or
trees containing arbitrary numbers of dictionary elements at each node (with nodes possibly containing no dictionary element).
More broadly, this formulation can be extended with the notion
of \emph{tree-structured} groups, which we now present:
\UpperSpace
\begin{definition}[Tree-structured set of groups.]\label{def:tree_struct}~\newline
~~A set of groups $\G \!\defin\! \{g\}_{g\in\G}$ is said to be tree-structured in $\{1,\dots,p\}$,
if
$
\, \bigcup_{g\in\G}\!g=\{1,\dots,p\}
$
and if for all $g,h \in \G$,
$
( g\cap h \neq \emptyset) \Rightarrow ( g \subseteq h\ \text{or}\ h \subseteq g ).
$
For such a set of groups, there exists a (non-unique) total order relation $\preceq$ such that:
$$
   g \preceq h  \; \Rightarrow \; \big\{ g \subseteq h \text{~~~or~~~} g \cap h = \emptyset\big\}.
$$
\end{definition}
\LowerSpace
Given such a tree-structured set of groups $\G$ and its associated norm $\Omega$,
we are interested throughout the paper in the following hierarchical sparse coding problem,
\begin{equation}\label{eq:main_formulation}
 \min_{\alphab \in \R{p}} f(\alphab) + \lambda \Omega(\alphab),
\end{equation}
where $\Omega$ is the tree-structured norm we have previously introduced,
the non-negative scalar $\lambda$ is a regularization parameter controlling the sparsity of the solutions of~(\ref{eq:main_formulation}),
and $f$ a smooth convex loss function
(see Section~\ref{sec:optimization} for more details about the smoothness assumptions on $f$).
In the rest of the paper, we will mostly use the square loss
$
f(\alphab) = \frac{1}{2} \|\x - \D\alphab\|^2_2,
$
with a dictionary $\D$ in $\Real^{m \times p}$, but
the formulation of Eq.~(\ref{eq:main_formulation}) extends beyond this context. In particular one
can choose $f$ to be the logistic loss, which is commonly used for
classification problems \cite[e.g., see][]{Hastie2009}.

Before turning to optimization methods for the hierarchical sparse coding problem, we consider a particular instance.
The \textit{sparse group Lasso} was recently considered by \citet{Sprechmann2010} and \citet{Friedman2010} as an extension of the
group Lasso of~\citet{Yuan2006}. To induce sparsity both groupwise and within groups, \citet{Sprechmann2010} and \citet{Friedman2010} add an $\ell_1$ term to the regularization of the group Lasso, which given a partition $\mathcal{P}$ of $\{1,\ldots,p\}$ in disjoint groups yields a regularized problem of the form
$$\min_{\alphab \in \Real^p} \dfrac{1}{2} \|\x - \D\alphab\|^2_2 + \lambda \sum_{g \in \mathcal{P}} \|\alphab_{\gi}\|_2 + \lambda' \|\alphab\|_1.$$
Since $\mathcal{P}$ is a partition, the set of groups in $\mathcal{P}$ and the singletons form together a tree-structured set of groups according to definition~\ref{def:tree_struct} and the algorithm we will develop is therefore applicable to this problem.

\subsubsection{Optimization for Hierarchical Sparsity-Inducing Norms}
While generic approaches like interior-point methods~\cite{Boyd2004} and
subgradient descent schemes~\cite{Bertsekas1999} might be used to deal with the
nonsmooth norm $\Omega$, several dedicated procedures have been proposed.

In \citet{Zhao2009}, a boosting-like technique is used,
with a path-following strategy in the specific case where $\|.\|$ is the $\ell_\infty$-norm.
Based on the variational equality
\begin{equation}\label{eq:eta_trick}
\|\u\|_1 = \min_{\z \in \R{p}_+} \dfrac{1}{2}\big[ \sum_{j=1}^p \dfrac{\u_j^2}{\z_j} + \z_j \big],
\end{equation}
\citet{Kim2009} follow a reweighted least-square scheme that is
well adapted to the square loss function. 
To the best of our knowledge, a formulation of this type is however not available
when $\|.\|$ is the $\ell_\infty$-norm. In addition it requires an appropriate smoothing to
become provably convergent.  The same approach is considered
by~\citet{Bach2008}, but built upon an active-set strategy.  Other proposed
methods consist of a projected gradient descent with approximate projections
onto the ball $\{\u \in \R{p}; \, \Omega(\u)\leq \lambda\}$~\cite{Schmidt2010},
and an augmented-Lagrangian based technique~\cite{Sprechmann2010} for
solving a particular case with two-level hierarchies.

While the previously listed first-order approaches are
(1) loss-function dependent, and/or
(2) not guaranteed to achieve optimal convergence rates, and/or
(3) not able to yield sparse solutions without a somewhat arbitrary post-processing step,
we propose to resort to proximal methods\footnote{Note that the authors of \citet{Chen2010} have considered proximal methods for general group structure $\G$ when $\|.\|$ is the $\ell_2$-norm;
due to a smoothing of the regularization term, the convergence rate they obtained is suboptimal.}
that do not suffer from any of these drawbacks.
\section{Optimization} \label{sec:optimization}

We begin with a brief introduction to proximal methods, necessary to present
our contributions.  From now on, we assume that $f$ is convex and continuously
differentiable with Lipschitz-continuous gradient.  
It is worth mentioning that there exist various proximal schemes in the literature 
that differ in their settings (e.g., batch versus stochastic) and/or the assumptions made on $f$.
For instance, the material we develop in this paper could also be applied to online/stochastic frameworks~\citep{Duchi2009, Hu2009, Xiao2010} 
and to possibly nonsmooth functions $f$~\citep[e.g.,][and references therein]{Duchi2009, Xiao2010, Combettes2010}.
Finally, most of the technical proofs of this section are presented in Appendix~\ref{sec:proofs} for
readability.

\subsection{Proximal Operator for the Norm $\Omega$}\label{sec:prox_op}
Proximal methods have drawn increasing attention in the signal processing
\cite[e.g.,][and numerous references therein]{Becker2009, Wright2009,
Combettes2010} and the machine learning communities \citep[e.g.,][and
references therein]{Bach2010}, especially because of their convergence rates
(optimal for the class of first-order techniques) and their ability to deal
with large nonsmooth convex problems~\citep[e.g.,][]{Nesterov2007, Beck2009}.
In a nutshell, these methods can be seen as a natural extension of gradient-based techniques
when the objective function to minimize has a nonsmooth part.
Proximal methods are iterative procedures. The simplest version of this
class of methods linearizes at each iteration the function $f$ around the current estimate
$\hat{\alphab}$, and this estimate is 
updated as the (unique by strong convexity) solution of the \textit{proximal problem}, defined as follows:
$$
\min_{\alphab \in \R{p}}~ f(\hat{\alphab}) + (\alphab - \hat{\alphab})^\top \nabla\! f(\hat{\alphab}) + \lambda \Omega(\alphab) + \frac{L}{2}\|\alphab - \hat{\alphab}\|_2^2.
$$
The quadratic term keeps the update in a neighborhood where $f$ is close to its linear approximation,
and $L\!>\!0$ is a parameter which is an upper bound on the Lipschitz constant of $\nabla f$.
This problem can be equivalently rewritten~as:

$$
\min_{\alphab \in \R{p}}~ \frac{1}{2}\Big\|\alphab - \big(\hat{\alphab} - \frac{1}{L}\nabla\! f(\hat{\alphab}) \big) \Big\|_2^2 + \frac{\lambda}{L} \Omega(\alphab).
$$
Solving \textit{efficiently} and \textit{exactly} this problem is crucial to
enjoy the fast convergence rates of proximal methods.
In addition, when the nonsmooth term $\Omega$ is not present, the previous
proximal problem exactly leads to the standard gradient update rule.
More generally, we define the \textit{proximal operator}: 
\UpperSpace
\begin{definition}[Proximal Operator]~\newline
The proximal operator associated with our regularization term $\lambda\Omega$, which we denote by $\text{Prox}_{\lambda \Omega}$, is the function that maps a vector $\u \in \R{p}$ to the unique solution of
\begin{equation}\label{eq:prox_problem}
   \min_{\v \in \R{p}} \frac{1}{2} \NormDeux{\u-\v}^2 + \lambda  \Omega(\v).
\end{equation}
\end{definition}
\LowerSpace
This operator was initially introduced by~\citet{Moreau1962} to generalize the 
projection operator onto a convex set. What makes proximal methods appealing for solving sparse
decomposition problems is that this operator can be often computed in closed-form. For instance,
\begin{itemize}
   \item When $\Omega$ is the $\ell_1$-norm---that is, $\Omega(\u)=\|\u\|_1$,
      the proximal operator is the well-known elementwise soft-thresholding
      operator,
      \begin{displaymath}
         \forall j \in \InSet{p},~~\u_j \mapsto \sign(\u_j)(|\u_j|-\lambda)_+~~~ = \begin{cases}
            0 & \text{if}~~ |\u_j| \leq \lambda \\
            \sign(\u_j)(|\u_j|-\lambda) & \text{otherwise}. 
         \end{cases}
      \end{displaymath}
   \item When $\Omega$ is a group-Lasso penalty with $\ell_2$-norms---that is,
      $\Omega(\u)=\sum_{g\in \G}\|\u_{\gi}\|_2$, with $\G$ being a partition of
      $\InSet{p}$, the proximal problem is \emph{separable} in every group, 
      and the solution is a generalization of the soft-thresholding operator to
      groups of variables:
      \begin{displaymath}
         \forall g\in \G~~, \u_{\gi} \mapsto \u_{\gi}-\Pi_{\|.\|_2 \leq \lambda}[\u_{\gi}] = \begin{cases}
            0 & \text{if}~~ \|\u_{\gi}\|_2 \leq \lambda \\
	    \frac{\|\u_{\gi}\|_2-\lambda}{\|\u_{\gi}\|_2}\u_{\gi} & \text{otherwise}, \\
         \end{cases} 
      \end{displaymath}
      where $\Pi_{\|.\|_2 \leq \lambda}$ denotes the orthogonal projection onto the ball of 
      the $\ell_2$-norm of radius $\lambda$.
   \item When $\Omega$ is a group-Lasso penalty with $\ell_\infty$-norms---that is,
      $\Omega(\u)=\sum_{g\in \G}\|\u_{\gi}\|_\infty$, the solution is also a group-thresholding
      operator:
      \begin{displaymath}
         \forall g\in \G,~~\u_{\gi} \mapsto \u_{\gi} - \Pi_{\|.\|_1 \leq \lambda}[\u_{\gi}],
      \end{displaymath}
      where $\Pi_{\|.\|_1 \leq \lambda}$ denotes the orthogonal projection onto
      the $\ell_1$-ball of radius $\lambda$, which can be solved in $O(p)$
      operations~\cite{Brucker1984,Maculan1989}. Note that when $\|\u_{\gi}\|_1 \leq \lambda$, we
      have a group-thresholding effect, with $\u_{\gi} - \Pi_{\|.\|_1 \leq \lambda}[\u_{\gi}]=0$.
\end{itemize}
More generally, a classical result~\cite[see, e.g.,][]{Combettes2010,Wright2009}
says that the proximal operator for a norm~$\|.\|$ can be computed as the
residual of the projection of a vector onto a ball of
the dual-norm denoted by $\DualNorm{.}$, and defined for any vector~$\kappab$ in $\Real^p$ by
$\DualNorm{\kappab} \defin \max_{\Norm{\z} \leq 1} \z^\top \kappab$.\footnote{It is
easy to show that the dual norm of the $\ell_2$-norm is the $\ell_2$-norm
itself. The dual norm of the $\ell_\infty$ is the $\ell_1$-norm.}
This is a classical duality result for proximal operators leading to the
different closed forms we have just presented. We have indeed that
$\text{Prox}_{\lambda\|.\|_2} = \text{Id}-\Pi_{\|.\|_2 \leq \lambda}$ and
$\text{Prox}_{\lambda\|.\|_\infty} = \text{Id}-\Pi_{\|.\|_1 \leq \lambda}$,
where $\text{Id}$ stands for the identity operator.
Obtaining closed forms is, however, not possible 
anymore as soon as some groups in $\G$ overlap,
which is always the case in our hierarchical setting with tree-structured groups.
\subsection{A Dual Formulation of the Proximal Problem}
We now show that Eq.~(\ref{eq:prox_problem}) can be solved using a dual approach, as
described in the following lemma. The result relies on conic duality
\cite{Boyd2004}, and does not make any assumption on the choice of the norm~$\|.\|$: 
\UpperSpace
\begin{lemma}[Dual of the proximal problem]\label{lem:dual}~\newline
   Let $\u \in \R{p}$ and let us consider the problem
\begin{equation}
   \begin{split}
  & \max_{ \xib \in \RR{p}{|\G|} } -\frac{1}{2} \Big[ \Big\| \u - \sum_{g \in \G} \xib^g  \Big\|^2_2 - \NormDeux{\u}^2 \Big] \\
  \mbox{ s.t. }& \forall g\in\G,\ \DualNorm{\xib^g} \leq \lambda \weights_g
  \mbox{ and } \, \xib^g_j = 0 \, \mbox{ if } \, j \notin g ,\label{eq:dual_problem}
  \end{split}
\end{equation}
where $\xib = (\xib^g)_{g \in \G}$ and $\xib^g_j$ denotes the $j$-th coordinate of the vector $\xib^g$ in $\R{p}$.
Then, problems~(\ref{eq:prox_problem}) and~(\ref{eq:dual_problem})
are dual to each other and strong duality holds.
In addition, the pair of primal-dual variables $\{\v,\xib\}$ is optimal if and only if
$\xib$ is a feasible point of the optimization problem~(\ref{eq:dual_problem}), and
 \begin{equation}\label{eq:opt_cond_primaldual}
 {\textstyle \v = \u -\sum_{g \in \G} \xib^g}\quad \mathrm{and}\quad 
  \forall g \in \G,\ \xib^g = \Pi_{\|.\|_\ast \leq \lambda \weights_g}(\v_{\gi} +\xib^g),
 \end{equation}
where we denote by $\Pi_{\|.\|_\ast \leq \lambda \weights_g}$ the orthogonal projection
onto the ball $\{\kappab\in\R{p};\ \DualNorm{\kappab} \leq \lambda \weights_g \}$.
\end{lemma}
\LowerSpace
Note that we focus here on
specific tree-structured groups, but the previous lemma is valid regardless of
the nature of~$\G$.
The rationale of introducing such a dual formulation is to consider an
equivalent problem to~(\ref{eq:prox_problem}) that removes the issue of
overlapping groups at the cost of a larger number of variables. In
Eq.~(\ref{eq:prox_problem}), one is indeed looking for a vector $\v$
of size $p$, whereas one is considering a matrix~$\xib$ in
$\Real^{p \times |\G|}$ in Eq.~(\ref{eq:dual_problem}) with
$\sum_{g\in \G}|g|$ nonzero entries, but with separable (convex) constraints for
each of its columns.
 
This specific structure makes it possible to use
block coordinate ascent~\cite{Bertsekas1999}.
Such a procedure is presented in Algorithm~\ref{alg:bcd}. It
optimizes sequentially Eq.~(\ref{eq:dual_problem}) with respect to the
variable~$\xib^g$, while keeping fixed the other variables $\xib^h$, for $h\neq g$.  
It is easy to see from Eq.~(\ref{eq:dual_problem}) that such an update of a column $\xib^g$, for a group~$g$ in~$\G$,
amounts to computing the orthogonal projection of the vector $\u_{\gi}-\sum_{h \neq g}\xi_{\gi}^h$
onto the ball of radius $\lambda \weights_g$ of the dual norm $\DualNorm{.}$.
\begin{algorithm}[!hbtp]
\caption{Block coordinate ascent in the dual}\label{alg:bcd}
\begin{algorithmic}
\STATE Inputs: $\u \in \R{p}$ and set of groups $\G$.
\STATE Outputs: $(\v,\xib)$ (primal-dual solutions).
\STATE Initialization: $\xib=\0$.
\WHILE{ ( \textit{maximum number of iterations not reached} ) }
\FOR{ $g \in \G$ }
\STATE $\xib^g \leftarrow \Pi_{\|.\|_\ast \leq \lambda \weights_g} (\big[ \u - \sum_{h \neq g} \xib^{h}\big]_{\gi})$.
\ENDFOR
\ENDWHILE
\STATE $\v \leftarrow  \u - \sum_{g \in \G} \xib^{g}$.
\end{algorithmic}
\end{algorithm}
\subsection{Convergence in One Pass}\label{sec:one_pass_conv}
In general, Algorithm~\ref{alg:bcd} is not guaranteed to solve exactly
Eq.~(\ref{eq:prox_problem}) in a finite number of iterations.  However, when
$\Norm{.}$ is the $\ell_2$- or $\ell_\infty$-norm, and provided that the groups
in $\G$ are appropriately ordered, we now prove that only \textit{one pass} of
Algorithm~\ref{alg:bcd}, i.e., only one iteration over all groups, is
sufficient to obtain the exact solution of Eq.~(\ref{eq:prox_problem}).
This~result constitutes the main technical contribution of the paper and is the
key for the efficiency of our procedure.

Before stating this result, we need to introduce a lemma showing that,
given two nested groups $g,h$ such that $g \subseteq h \subseteq \{1,\dots, p\}$,
if $\xib^g$ is updated before $\xib^h$ in Algorithm~\ref{alg:bcd}, then
the optimality condition for $\xib^g$ is not perturbed by the update of $\xib^h$.
\UpperSpace
\begin{lemma}[Projections with nested groups]\label{lem:proj_nested_gr}~\\
Let $\Norm{.}$ denote either the $\ell_2$- or $\ell_\infty$-norm, and
$g$ and $h$ be two nested groups---that is, $g \subseteq h \subseteq \{1,\dots, p\}$.
Let $\u$ be a vector in $\R{p}$, and let us consider the successive projections
$$
\xib^g \defin \Pi_{\|.\|_\ast \leq t_g}(\u_{\gi})
	~\text{ and }~
        \xib^h \defin \Pi_{\|.\|_\ast \leq t_h}(\u_{\hi}-\xib^g),
$$
with $t_g, t_h > 0$. Let us introduce $\v = \u - \xib^g - \xib^h$.
The following relationships hold
$$\xib^g = \Pi_{\|.\|_\ast \leq t_g}(\v_{\gi} +\xib^g)\quad \text{and}\quad 
\xib^h = \Pi_{\|.\|_\ast \leq t_h}(\v_{\hi} +\xib^h).$$
\end{lemma}
\LowerSpace
The previous lemma establishes the convergence in one pass of
Algorithm~\ref{alg:bcd} in the case where $\G$ only contains two nested groups
$g \subseteq h$, provided that $\xib^g$ is computed before~$\xib^h$. Let us
illustrate this fact more concretely.  After initializing $\xib^g$ and
$\xib^h$ to zero, Algorithm~\ref{alg:bcd} first updates $\xib^g$ with the
formula $\xib^g \leftarrow \Pi_{\|.\|_\ast \leq \lambda \weights_g}(\u_{\gi})$, and then performs
the following update: $\xib^h \leftarrow \Pi_{\|.\|_\ast \leq \lambda
\weights_h}(\u_{\hi}-\xib^g)$ (where we have used that $\xib^g = \xib_{\hi}^g$ since $g
\subseteq h$). We are now in position to apply Lemma~\ref{lem:proj_nested_gr}
which states that the primal/dual variables $\{\v,\xib^g,\xib^h\}$
satisfy the optimality conditions~(\ref{eq:opt_cond_primaldual}), as described in Lemma~\ref{lem:dual}.
In only one pass over the groups $\{g,h\}$, we have in fact reached a
solution of the dual formulation presented in Eq.~(\ref{eq:dual_problem}), 
and in particular, the solution of the proximal problem~(\ref{eq:prox_problem}).

In the following proposition, this lemma is extended to general tree-structured
sets of groups $\G$:
\UpperSpace
\begin{proposition}[Convergence in one pass]\label{prop:one_pass_conv}~\\
        Suppose that the groups in $\G$ are ordered according to the total
order relation $\preceq$ of Definition~\ref{def:tree_struct}, and
that the norm $\Norm{.}$ is either the $\ell_2$- or $\ell_\infty$-norm.  Then,
after initializing $\xib$ to $\0$, \textit{a single pass} of Algorithm~\ref{alg:bcd}
over $\G$ with the order~$\preceq$ yields the solution of the proximal
problem~(\ref{eq:prox_problem}).
\end{proposition}
\begin{proof}
The proof largely relies on \refLemma{lem:proj_nested_gr} and proceeds by
induction.
By definition of Algorithm~\ref{alg:bcd},
the feasibility of $\xib$ is always guaranteed.
We consider the following induction hypothesis
$$
\mathcal{H}(h) \defin \big\{ \forall g \preceq h,\ \text{it holds that}\
\xib^g = \Pi_{\|.\|_\ast \leq \lambda \weights_g}([\u-{\textstyle\sum_{g' \preceq h}} \xib^{g'}]_{\gi} + \xib^{g} )  \big\}.
$$
Since the dual variables $\xib$ are initially equal to zero,
the summation over $g' \preceq h,\ g'\neq g$ is equivalent to a summation over $g'\neq g$.
We initialize the induction with the first group in $\G$, that, by definition
of $\preceq$, does not contain any other group.
The first step of Algorithm~\ref{alg:bcd} easily shows that the induction hypothesis $\mathcal{H}$ is satisfied
for this first group.

We now assume that $\mathcal{H}(h)$ is true and consider the next group $h'$, $h \preceq h'$,
in order to prove that $\mathcal{H}(h')$ is also satisfied.
We have for each group $g \subseteq h$,
$$
   \xib^g = \Pi_{\|.\|_\ast \leq \lambda \weights_g}([\u-{\textstyle\sum_{g' \preceq h}} \xib^{g'}]_{\gi} + \xib^{g} )
          = \Pi_{\|.\|_\ast \leq \lambda \weights_g}([\u-{\textstyle\sum_{g' \preceq h}} \xib^{g'} + \xib^{g}]_{\gi} ).
$$
Since $\xib^{g}_{\hi'}=\xib^{g}$ for $g \subseteq h'$, we have
$$
[\u-{\textstyle\sum_{g' \preceq h}} \xib^{g'}]_{\hi'} = [\u-{\textstyle\sum_{g' \preceq h}} \xib^{g'}]_{\hi'}+ \xib^{g} - \xib^{g}
= [\u-{\textstyle\sum_{g' \preceq h}} \xib^{g'}+ \xib^{g}]_{\hi'} - \xib^{g},
$$
and following the update rule for the group $h'$,
$$
   \xib^{h'} = \Pi_{\|.\|_\ast \leq \lambda \weights_{h'}}([\u-{\textstyle\sum_{g' \preceq h}} \xib^{g'}]_{\hi'})
             = \Pi_{\|.\|_\ast \leq \lambda \weights_{h'}}([\u-{\textstyle\sum_{g' \preceq h}} \xib^{g'}+ \xib^{g}]_{\hi'} - \xib^{g}).
$$
At this point, we can apply \refLemma{lem:proj_nested_gr} for each group $g \subseteq h$, which proves
that the induction hypothesis $\mathcal{H}(h')$ is true.
Let us introduce $\v \defin \u - \sum_{g\in\G}\xib^g$.
We have shown that for all $g$ in $\G$,
$
\xib^{g} = \Pi_{\|.\|_\ast \leq \lambda \weights_{g}}( \v_\gi + \xib^{g} ).
$
As a result, the pair $\{\v,\xib\}$ satisfies the optimality conditions~(\ref{eq:opt_cond_primaldual}) 
of problem~(\ref{eq:dual_problem}). Therefore, after one complete pass over $g\in\G$,
the primal/dual pair $\{\v,\xib\}$ is optimal, and in particular, 
$\v$ is the solution of problem~(\ref{eq:prox_problem}).
\end{proof}
Using conic duality, we have derived a dual formulation of the proximal
operator, leading to Algorithm~\ref{alg:bcd} which is generic and works for any
norm $\|.\|$, as long as one is able to perform projections onto balls of the
dual norm~$\|.\|_\ast$. We have further shown that when $\|.\|$ is the $\ell_2$-
or the $\ell_\infty$-norm, a single pass provides the exact solution
when the groups $\G$ are correctly ordered.  We show however in Appendix~\ref{appendix:counter}, that, perhaps surprisingly, the conclusions of Proposition~\ref{prop:one_pass_conv} do not hold for general $\ell_q$-norms, if $q\notin \{1,2,\infty\}$.
Next, we give another interpretation of
this result.

\subsection{Interpretation in Terms of Composition of Proximal Operators}
In Algorithm~\ref{alg:bcd}, since all the vectors $\xib^g$ are initialized to $\0$, 
when the group $g$ is considered, we have by induction $\u-\sum_{h\neq g} \xib^h=\u-\sum_{h \preceq g} \xib^h$. Thus,
to maintain at each iteration of the inner loop $\v=\u-\sum_{h\neq g}\xib^h$ one can instead update $\v$ after updating $\xib^g$ according to $\v \leftarrow \v-\,\xib^g$. Moreover, since~$\xib^g$ is no longer needed in the algorithm, and since only the entries of $\v$ indexed by $g$ are updated, we can combine the two updates into $\v_{\gi} \leftarrow \v_{\gi} - \Pi_{\|.\|_\ast \leq \lambda \weights_g} (\v_{\gi})$,
leading to a simplified Algorithm~\ref{alg:bcd2} equivalent to Algorithm~\ref{alg:bcd}.

\begin{algorithm}[!hbtp]
\caption{Practical Computation of the Proximal Operator for $\ell_2$- or $\ell_\infty$-norms.}\label{alg:bcd2}
\begin{algorithmic}
\STATE Inputs: $\u \in \R{p}$ and an ordered tree-structured set of groups $\G$.
\STATE Outputs: $\v$ (primal solution).
\STATE Initialization: $\v=\u$.
\FOR{ $g \in \G$, following the order $\preceq$, }
\STATE $\displaystyle \v_{\gi} \leftarrow \v_{\gi} - \Pi_{\|.\|_\ast \leq \lambda \weights_g} (\v_{\gi})$.
\ENDFOR
\end{algorithmic}
\end{algorithm}
Actually, in light of the classical relationship between proximal operator and projection 
(as discussed in Section~\ref{sec:prox_op}), 
it is easy to show that each update $\v_{\gi} \leftarrow \v_{\gi} -
\Pi_{\|.\|_\ast \leq \lambda \weights_g} (\v_{\gi})$ is equivalent to $\v_{\gi}
\leftarrow \text{Prox}_{\lambda \weights_g\|.\|}[\v_{\gi}]$.
To simplify the notations, we define the proximal operator for a group $g$ in $\G$ as
$\text{Prox}^{g}(\u) \defin  \text{Prox}_{\lambda\weights_g\|.\|}(\u_{\gi})$ for
every vector $\u$ in $\Real^p$.

Thus, Algorithm~\ref{alg:bcd2} in fact performs a sequence of $|\G|$ proximal
operators, and we have shown the following corollary of Proposition~\ref{prop:one_pass_conv}:
\UpperSpace
\begin{corollary}[Composition of Proximal Operators]\label{cor:composition}~\\
Let $g_1 \preccurlyeq \ldots \preccurlyeq g_m$ such that $\G=\{g_1,\ldots,g_m\}$. The proximal operator $\text{Prox}_{\lambda \Omega}$ associated with the norm $\Omega$ can be written as the composition of elementary operators:
$$\text{Prox}_{\lambda \Omega}=\text{Prox}^{g_{m}} \circ \ldots \circ \text{Prox}^{g_{1}}.$$
\end{corollary}
\LowerSpace

\subsection{Efficient Implementation and Complexity}\label{sec:efficiency}
Since Algorithm~\ref{alg:bcd2} involves $|\G|$
projections on the dual balls (respectively the $\ell_2$- and the $\ell_1$-balls
for the $\ell_2$- and $\ell_\infty$-norms) of vectors in $\Real^p$, in a first approximation, its complexity is at most $O(p^2)$, because each
of these projections can be computed in $O(p)$
operations~\cite{Brucker1984,Maculan1989}.
But in fact, the algorithm performs one projection for each group $g$ involving $|g|$ variables, and
the total complexity is therefore $O\Big(\sum_{g \in \G}|g|\Big)$.
By noticing that if $g$ and~$h$ are two groups with the same depth in the tree,
then $g \cap h = \emptyset$, it is easy to show that the number of variables involved in all the projections
is less than or equal to $dp$, where $d$ is the depth of the tree:
\UpperSpace
\begin{lemma}[Complexity of Algorithm \ref{alg:bcd2}]\label{lemma:linfty}~\newline
   Algorithm \ref{alg:bcd2} gives the solution of the primal problem Eq.~(\ref{eq:prox_problem}) in $O(pd)$ operations,
   where $d$ is the depth of the tree.
\end{lemma}
\LowerSpace
Lemma \ref{lemma:linfty} should not suggest that the complexity is linear in $p$, 
since $d$ could depend of $p$ as well, 
and in the worst case the hierarchy is a chain, yielding $d=p-1$. 
However, in a balanced tree, $d=O(\log(p))$. In practice, the structures we have considered experimentally are relatively flat, with a depth not exceeding $d=5$, and the complexity is therefore almost linear.

Moreover, in the case of the $\ell_2$-norm, it is actually possible to propose
an algorithm with complexity $O(p)$. Indeed, in that case each of the proximal
operators $\Prox^g$ is a scaling operation: $\v_{\gi} \leftarrow \big(1-
\lambda  \weights_g / \|\v_{\gi}\|_2 \big )_+ \v_{\gi}$. The composition of these
operators in Algorithm~\ref{alg:bcd} thus corresponds to performing sequences
of scaling operations.  The idea behind Algorithm~\ref{alg:fbcd} is that the
corresponding scaling factors depend only on the norms of the successive
residuals of the projections and that these norms can be computed recursively
in one pass through all nodes in $O(p)$ operations; finally, computing and
applying all scalings to each entry takes then again $O(p)$ operations.

To formulate the algorithm, two new notations are used: for a group $g$ in $\G$, we
denote by $\text{root}(g)$ the indices of the variables that are at the root of the
subtree corresponding to~$g$,\footnote{As a reminder, $\text{root}(g)$ 
is not a singleton when several dictionary elements are considered per node.} 
and by $\text{children}(g)$ the set of groups that are
the children of $\text{root}(g)$ in the tree. For example, in the
tree presented in Figure~\ref{fig:groups}, $\text{root}(\{3,5,6\})\!=\!\{3\}$, $\text{root}(\{1,2,3,4,5,6\})\!=\!\{1\}$,
$\text{children}(\{3,5,6\})\!=\!\{ \{5\},\{6\} \}$, and $\text{children}(\{1,2,3,4,5,6\})\!=\!\{ \{2,4\},\{3,5,6\} \}$.
Note that all the groups of
$\text{children}(g)$ are necessarily included in $g$.
\begin{algorithm}[t]
   \caption{Fast computation of the Proximal operator for $\ell_2$-norm case.}\label{alg:fbcd}
\begin{algorithmic}[1]
\REQUIRE $\u \in \R{p}$ (input vector), set of groups $\G$, $(\weights_g)_{g \in \G}$
(positive weights), and $g_0$ (root of the tree).
\STATE Variables: $\rhob=(\rho_g)_{g \in \G}$ in $\R{|\G|}$ (scaling factors); $\v$ in $\R{p}$ (output, primal variable).
\STATE \texttt{computeSqNorm}($g_0$).
\STATE \texttt{recursiveScaling}($g_0$,$1$).
\RETURN $\v$ (primal solution).
\end{algorithmic}
\vspace*{0.2cm}
{\bf Procedure} \texttt{computeSqNorm}($g$)
\begin{algorithmic}[1]
   \STATE Compute the squared norm of the group: $\eta_g \leftarrow \|\u_{\text{root}(g)}\|_2^2 + \sum_{h \in \text{children}(g)}
   \text{\texttt{computeSqNorm}}(h)$.
   \STATE Compute the scaling factor of the group: $\rho_g \leftarrow \big (1-\lambda \weights_g/ \sqrt{\eta_g} \big )_+$.
   \RETURN $\eta_g\rho_g^2$.
\end{algorithmic}
{\bf Procedure} \texttt{recursiveScaling}($g$,$t$)
\begin{algorithmic}[1]
   \STATE $\rho_g \leftarrow t\rho_g$.
   \STATE $\v_{\text{root}(g)} \leftarrow \rho_g \u_{\text{root}(g)}$.
   \FOR{$h \in \text{children}(g)$}
   \STATE  \texttt{recursiveScaling}($h$,$\rho_g$).
   \ENDFOR
\end{algorithmic}
\end{algorithm}
The next lemma is proved in Appendix~\ref{sec:proofs}.
\UpperSpace
\begin{lemma}[Correctness and complexity of Algorithm \ref{alg:fbcd}]\label{lemma:l2}~\newline
   When $\Norm{.}$ is chosen to be the $\ell_2$-norm, Algorithm \ref{alg:fbcd} gives the solution of the primal problem Eq.~(\ref{eq:prox_problem}) in $O(p)$ operations.
\end{lemma}
\LowerSpace
So far the dictionary $\D$ was fixed to be for example a wavelet basis.
In the next section, we apply the tools we developed for solving efficiently problem~(\ref{eq:main_formulation}) to learn a dictionary $\D$ adapted to our hierarchical sparse coding formulation.

\section{Application to Dictionary Learning} \label{sec:dictionary_learning}

We start by briefly describing dictionary learning.
\subsection{The Dictionary Learning Framework}
Let us consider a set
$ \X = [\x^1,\dots,\x^n]$ in $\RR{m}{n}$
of $n$ signals of dimension $m$.
Dictionary learning is a matrix factorization problem which aims at representing these
signals as linear combinations of the dictionary elements,
that are the columns of a matrix $\D = [\d^1,\dots,\d^p]$ in $\RR{m}{p}$.
More precisely, the dictionary $\D$ is \textit{learned} along with a matrix of
decomposition coefficients $\A = [\alphab^1,\dots,\alphab^n]$ in~$\RR{p}{n}$, so that $\x^i \approx \D \alphab^i$ for every signal~$\x^i$.

While learning simultaneously $\D$ and $\A$, one may want to encode specific
prior knowledge about the problem at hand, such as, for example, the positivity of the
decomposition \cite{Lee1999}, or the sparsity of~$\A$~\cite{Olshausen1997,Aharon2006,Lee2007,Mairal2010}.
This leads to penalizing or constraining
$(\D,\A)$ and results in the following formulation:
\begin{equation}\label{eq:general_formulation}
   \min_{\D \in \mathcal{D},\A \in \mathcal{A}}
   \frac{1}{n}\sum_{i=1}^n\Big[\frac{1}{2} \|\x^i-\D\alphab^i\|_2^2 + \lambda \Psi(\alphab^i) \Big],
\end{equation}
where $\mathcal{A}$ and $\mathcal{D}$ denote two convex sets and $\Psi$ is a regularization term,
usually a norm or a squared norm, whose effect is controlled by the regularization parameter $\lambda > 0$.
Note that $\mathcal{D}$ is assumed to be bounded to avoid any degenerate solutions of Problem~(\ref{eq:general_formulation}).
For instance, the standard sparse coding formulation
takes $\Psi$ to be the $\ell_1$-norm,
$\mathcal{D}$ to be the set of matrices in $\RR{m}{p}$ whose columns
have unit $\ell_2$-norm,
with $\mathcal{A}=\Real^{p \times n}$~\cite{Olshausen1997,Lee2007,Mairal2010}.

However, this classical setting treats each dictionary element independently
from the others, and does not exploit possible relationships between them.  
To embed the dictionary in a tree structure, we therefore replace the $\ell_1$-norm
by our hierarchical norm and set $\Psi=\Omega$ in Eq.~(\ref{eq:general_formulation}).

A question of interest is whether hierarchical priors are more appropriate in supervised settings 
or in the matrix-factorization context in which we use it. It is not so common
in the supervised setting to have strong prior information that allows us to organize the features in a hierarchy.
On the contrary, in the case of dictionary learning, since the atoms are learned, one can argue that 
the dictionary elements learned will \emph{have to} match well the hierarchical prior that is imposed by the regularization.
In other words, combining structured regularization with dictionary learning has precisely
the advantage that the dictionary elements will \emph{self-organize} to match the prior.

\subsection{Learning the Dictionary} \label{subsec:updateD}

Optimization for dictionary learning has already been intensively studied. We choose in this paper
a typical alternating scheme, which optimizes in turn $\D$ and $\A=[\alphab^1,\ldots,\alphab^n]$ while keeping
the other variable fixed~\cite{Aharon2006,Lee2007,Mairal2010}.\footnote{Note that although we use this classical scheme for simplicity, it would also be possible to use the stochastic approach proposed by~\citet{Mairal2010}.}
Of course, the convex optimization tools we develop in this paper
do not change the intrinsic non-convex nature of the dictionary learning problem.
However, they solve the underlying convex subproblems efficiently, which is crucial to yield good results in practice.
In the next section, we report good performance on some applied problems, and 
we show empirically that
our algorithm is stable and does not seem to get trapped in bad local minima.
The main difficulty of our problem lies in the
optimization of the vectors~$\alphab^i$, $i$ in $\{1,\ldots,n\}$, for the dictionary $\D$ kept fixed.
Because of $\Omega$, the corresponding convex subproblem is nonsmooth
and has to be solved for each of the $n$ signals considered.
The optimization of the dictionary $\D$ (for~$\A$ fixed), which we discuss first, is in general easier.
\paragraph{Updating the dictionary $\D$.}
We follow the matrix-inversion free procedure
of~\citet{Mairal2010} to update the dictionary.
This method consists in iterating block-coordinate descent over the columns of $\D$.
Specifically, we assume that the domain set $\mathcal{D}$ has the
form
\begin{equation}\label{eq:d_mu}
\mathcal{D}_\mu \defin \{ \D\in\RR{m}{p},\ \mu \|\d^j\|_1+ (1-\mu)\| \d^j\|_2^2 \leq 1, ~\text{for all}~ j \in \{1,\ldots,p\} \},
\end{equation}
or $
\mathcal{D}_\mu^+ \defin \mathcal{D}_\mu \cap \RR{m}{p}_+,
$
with $\mu \in [0,1]$.  The choice for these particular domain sets is motivated
by the experiments of Section~\ref{sec:experiment}.  For natural image patches,
the dictionary elements are usually constrained to be in the unit $\ell_2$-norm
ball (i.e., $\DD=\DD_0$), while for topic modeling, the dictionary elements are
distributions of words and therefore belong to the simplex (i.e.,
$\DD=\DD_1^+$).  The update of each dictionary element amounts to performing a
Euclidean projection, which can be computed efficiently~\cite{Mairal2010}.
Concerning the stopping criterion, we follow the strategy from the same authors
and go over the columns of~$\D$ only a few times, typically $5$ times in our
experiments.
Although we have not explored locality constraints on the dictionary elements, 
these have been shown to be particularly relevant to some applications such as patch-based image classification~\citep{Yu2009}.
Combining tree structure and locality constraints is an interesting future research.

\paragraph{Updating the vectors $\alphab^i$.}
The procedure for updating the columns of $\A$ is based on the results
derived in Section~\ref{sec:one_pass_conv}.  Furthermore, positivity
constraints can be added on the domain of $\A$, by noticing that for our norm~$\Omega$
and any vector $\u$ in $\R{p}$, adding these constraints when
computing the proximal operator is equivalent to solving
$
\min_{\v \in \R{p}} \frac{1}{2}\! \NormDeux{[\u]_+ \! - \v}^2 + \lambda  \Omega(\v).
$
This equivalence is proved in Appendix~\ref{sec:nonneg_prox}.
We will indeed use positive decompositions to model text corpora in Section \ref{sec:experiment}.
Note that by constraining the decompositions $\alphab^i$ to be nonnegative, some entries $\alphab^i_j$ may be set to zero in addition to those already zeroed out by the norm~$\Omega$. 
As a result, the sparsity patterns obtained in this way might not satisfy the tree-structured condition~(\ref{eq:ancestor_cond}) anymore.
\section{Experiments} \label{sec:experiment}

We next turn to the experimental validation of our hierarchical sparse coding.
\subsection{Implementation Details}
In Section~\ref{sec:one_pass_conv}, we have shown that the proximal operator
associated to $\Omega$ can be computed exactly and efficiently. The problem is therefore amenable to
fast proximal algorithms that are well suited to nonsmooth convex
optimization.
Specifically, we tried the accelerated scheme from both
\citet{Nesterov2007} and \citet{Beck2009}, and finally opted for the latter
since, for a comparable level of precision, fewer calls of the proximal
operator are required.  The basic proximal scheme presented in \refSec{sec:prox_op}
is formalized by \citet{Beck2009} as an algorithm called ISTA; the same authors propose moreover
an accelerated variant, FISTA, which is a similar procedure, except that the
operator is not directly applied on the current estimate, but on an auxiliary
sequence of points that are linear combinations of past estimates.  This latter algorithm has
an optimal convergence rate in the class of first-order techniques, and also
allows for warm restarts, which is crucial in the alternating scheme of dictionary
learning.\footnote{Unless otherwise specified, 
the initial stepsize in ISTA/FISTA is chosen as the maximum eigenvalue of the sampling covariance matrix divided by 100, while 
the growth factor in the line search is set to $1.5$.}

Finally, we monitor the convergence of the algorithm by checking the relative
decrease in the cost function.\footnote{We are currently investigating algorithms for computing duality gaps based on network flow optimization tools~\cite{Mairal2010a}.}
Unless otherwise specified, all the algorithms used in the following
experiments are implemented in \texttt{C/C++}, with a \texttt{Matlab}
interface. Our implementation is freely available at \texttt{http://www.di.ens.fr/willow/SPAMS/}.

\subsection{Speed Benchmark}
To begin with, we conduct speed comparisons between our approach and other
convex programming methods, in the setting where $\Omega$ is chosen to be a
linear combination of $\ell_2$-norms.  The algorithms that take part in the
following benchmark are:\\
\hspace*{0.2cm}\textbullet\ Proximal methods, with ISTA and the accelerated FISTA methods~\cite{Beck2009}.\\
\hspace*{0.2cm}\textbullet\ A reweighted-least-square scheme (Re-$\ell_2$), as described by~\citet{Jenatton2009, Kim2009}.
This approach is adapted to the square loss, 
since closed-form updates can be used.\footnote{The computation of the updates related to the variational 
formulation~(\ref{eq:eta_trick}) also benefits from the hierarchical structure of $\G$, and can be performed in $O(p)$ operations.}\\
\hspace*{0.2cm}\textbullet\ Subgradient descent, whose step size is taken to be equal either to
$a/(k+b)$ or $a/(\sqrt{k}+b)$ (respectively referred to as SG and $\text{SG}_{\text{sqrt}}$), where $k$ is the iteration number, and $(a,b)$ are the best\footnote{``The best step size'' is understood as being the step size leading to the smallest cost function after 500 iterations.} parameters selected on the logarithmic grid
$(a,b) \in \{10^{-4},\ldots,10^3\}\times \{10^{-2}, \ldots, 10^5\}$.\\
\hspace*{0.2cm}\textbullet\ A commercial software (\texttt{Mosek}, available at \texttt{http://www.mosek.com/}) for second-order cone programming (SOCP).\\
Moreover, the experiments we carry out cover various settings, 
with notably different sparsity regimes, i.e., low, medium and high, respectively
corresponding to about $50\%, 10\%$ and $1\%$ of the total number of dictionary elements.
Eventually, all reported results are obtained on a single core of a 3.07Ghz CPU with 8GB of memory.
\begin{figure}[h!]
  \centering
   \subfloat[scale: small, regul.: low]{\includegraphics[width=0.33\linewidth]{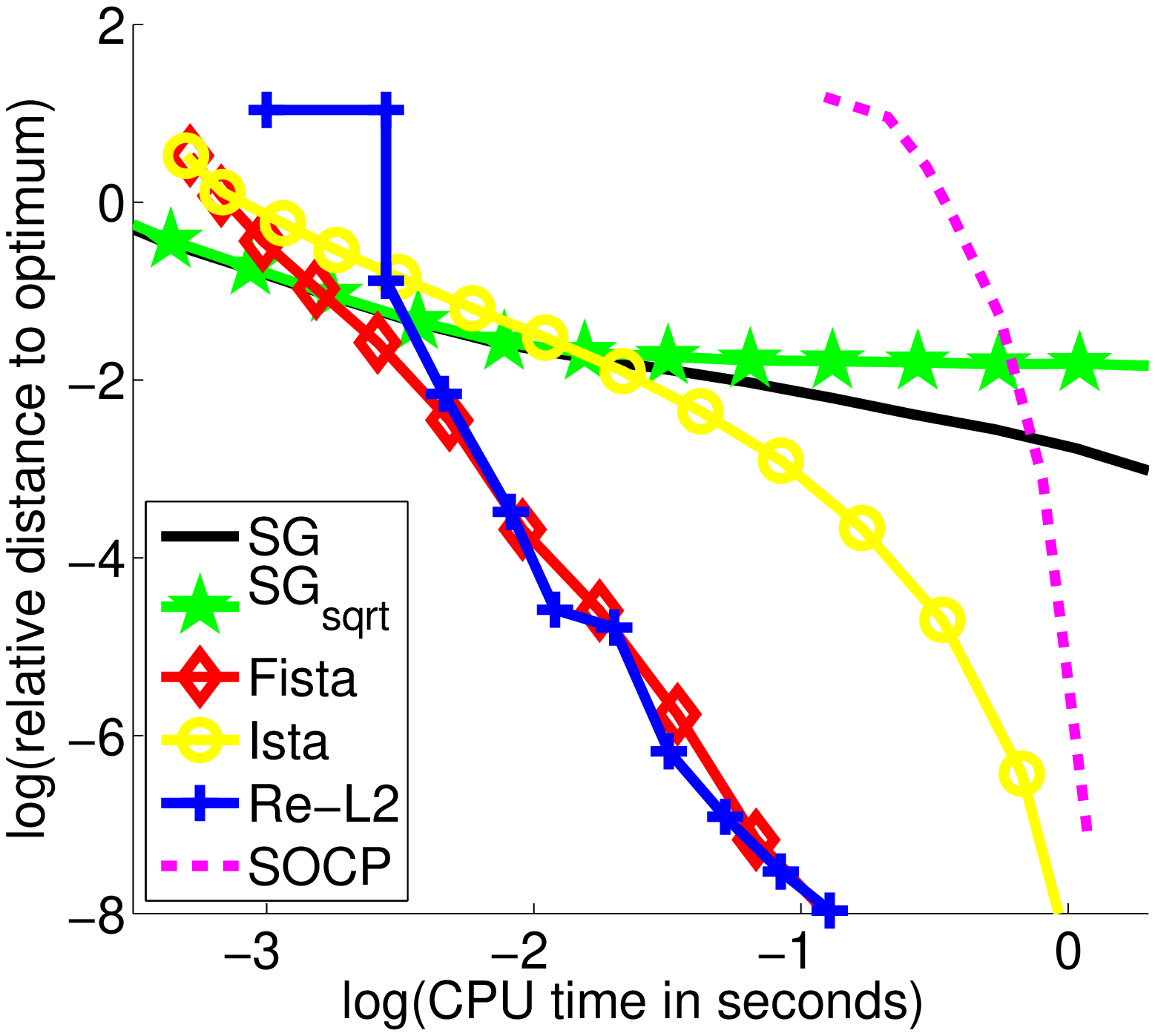}} \hfill
   \subfloat[scale: small, regul.: medium]{\includegraphics[width=0.33\linewidth]{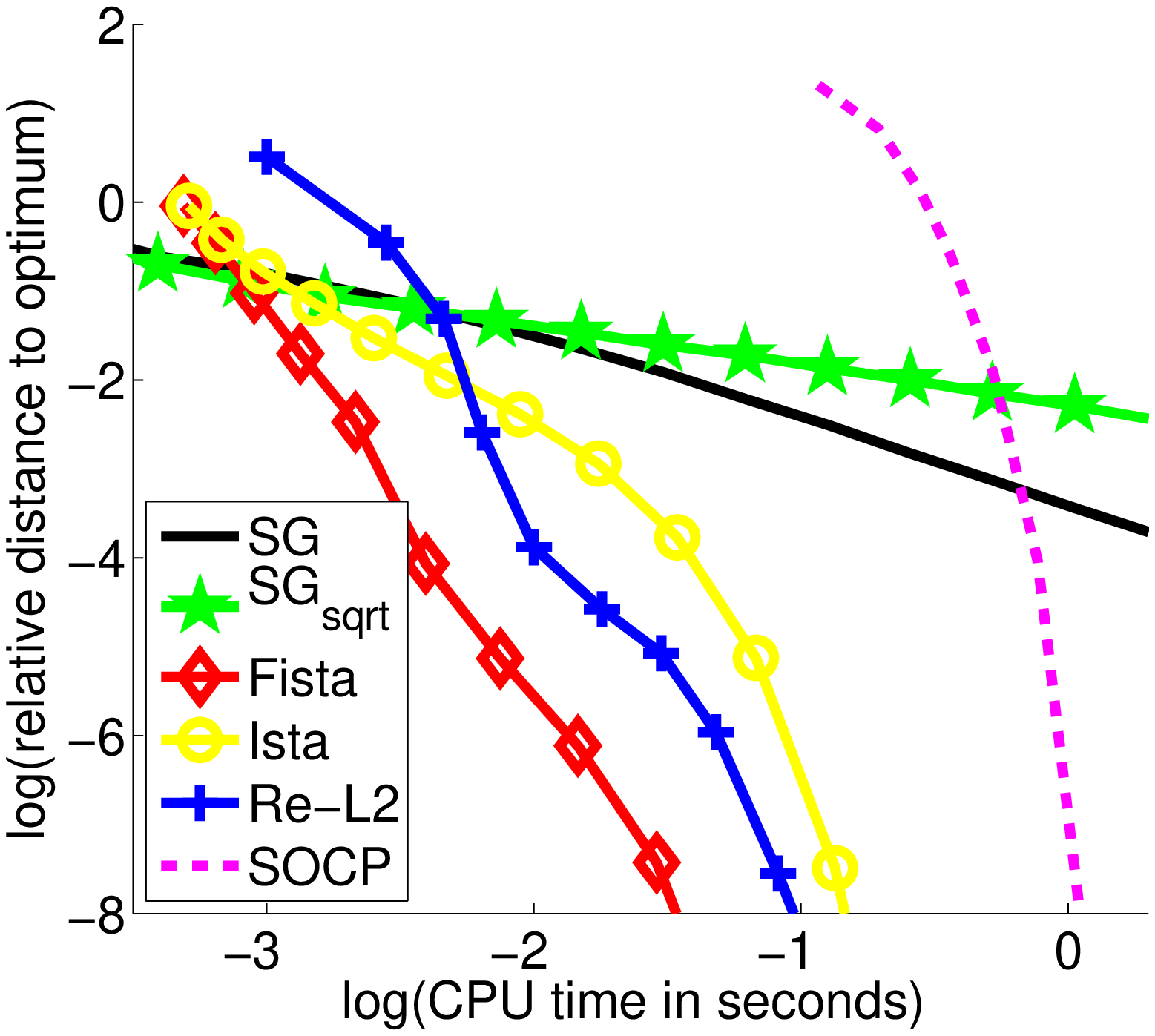}} \hfill
   \subfloat[scale: small, regul.: high]{\includegraphics[width=0.33\linewidth]{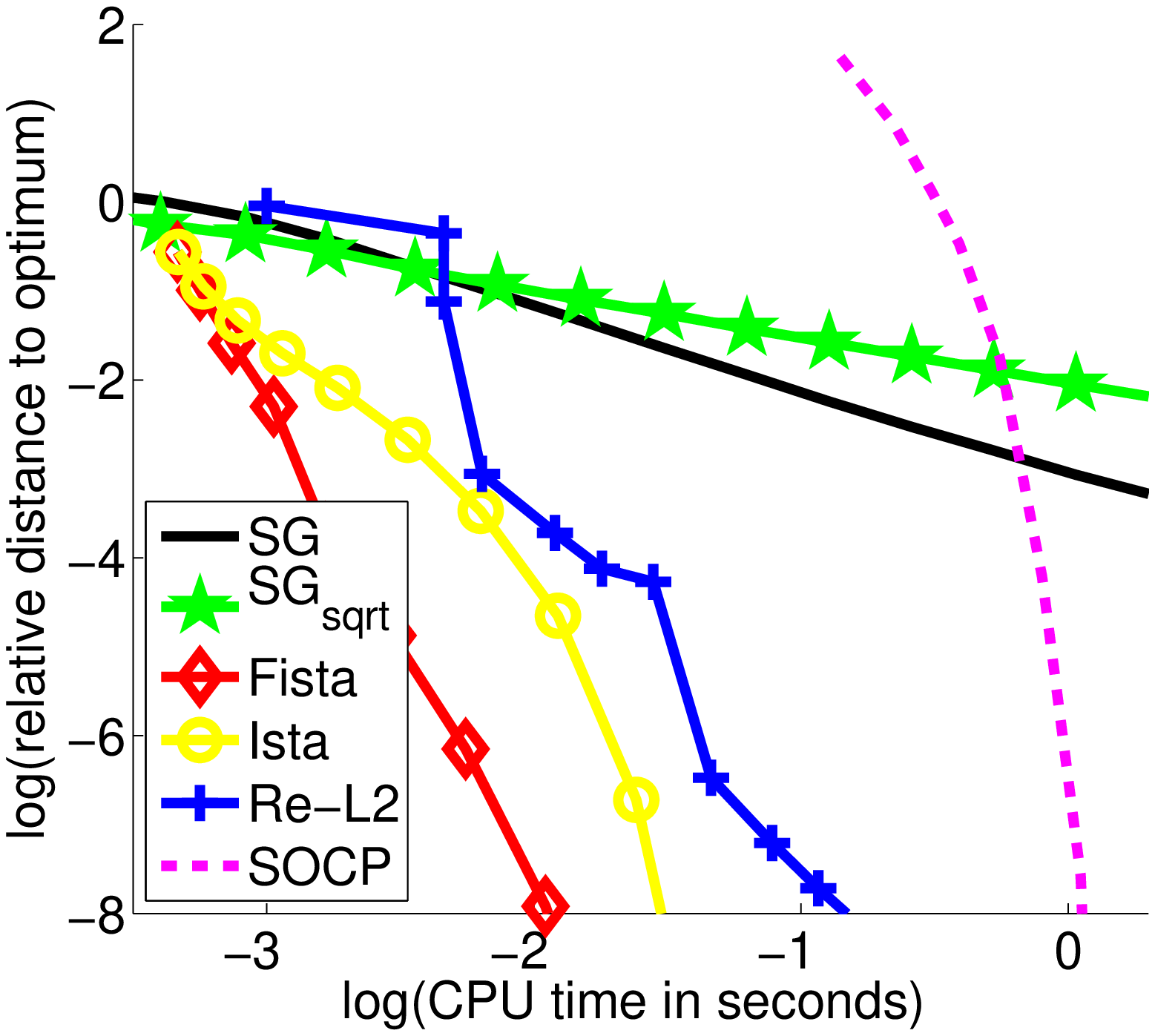}} \\
   \caption{Benchmark for solving a least-squares regression problem regularized by the hierarchical norm $\Omega$.
   The experiment is small scale, $m=256, p=151$, and shows the performances of six optimization methods (see main text for details)
   for three levels of regularization. The curves represent the relative value of the objective to the optimal value as a function of the computational time in second on a $\log_{10}/\log_{10}$ scale. All reported results are obtained by averaging $5$ runs.}
   \label{fig:struct_bench_patches}
\end{figure}
\subsubsection{Hierarchical dictionary of natural image patches}
In this first benchmark, we consider a least-squares regression problem regularized by $\Omega$ that arises in the context of
denoising of natural image patches, as further exposed in Section~\ref{sec:natural_img_patches}.
In particular, based on a hierarchical dictionary, 
we seek to reconstruct noisy $16\!\times\!16$-patches.
The dictionary we use is represented on Figure~\ref{fig:tree2}.
Although the problem involves a small number of variables, i.e., $p=151$ dictionary elements,
it has to be solved repeatedly for tens of thousands of patches, at moderate precision.
It is therefore crucial to be able to solve this problem quickly and efficiently.

We can draw several conclusions from the results of the simulations reported in Figure~\ref{fig:struct_bench_patches}.
First, we observe that in most cases,
the accelerated proximal scheme performs better than the other approaches.
In addition, unlike FISTA, ISTA seems to suffer in non-sparse scenarios.
In the least sparse setting,
the reweighted-$\ell_2$ scheme is the only method that competes with FISTA.
It is however not able to yield truly sparse solutions, and would therefore need a subsequent (somewhat arbitrary) thresholding operation.
As expected, the generic techniques such as SG and SOCP do not compete with dedicated algorithms.
\begin{figure}[h!]
  \subfloat[scale: large, regul.: low]{\includegraphics[width=0.33\linewidth]{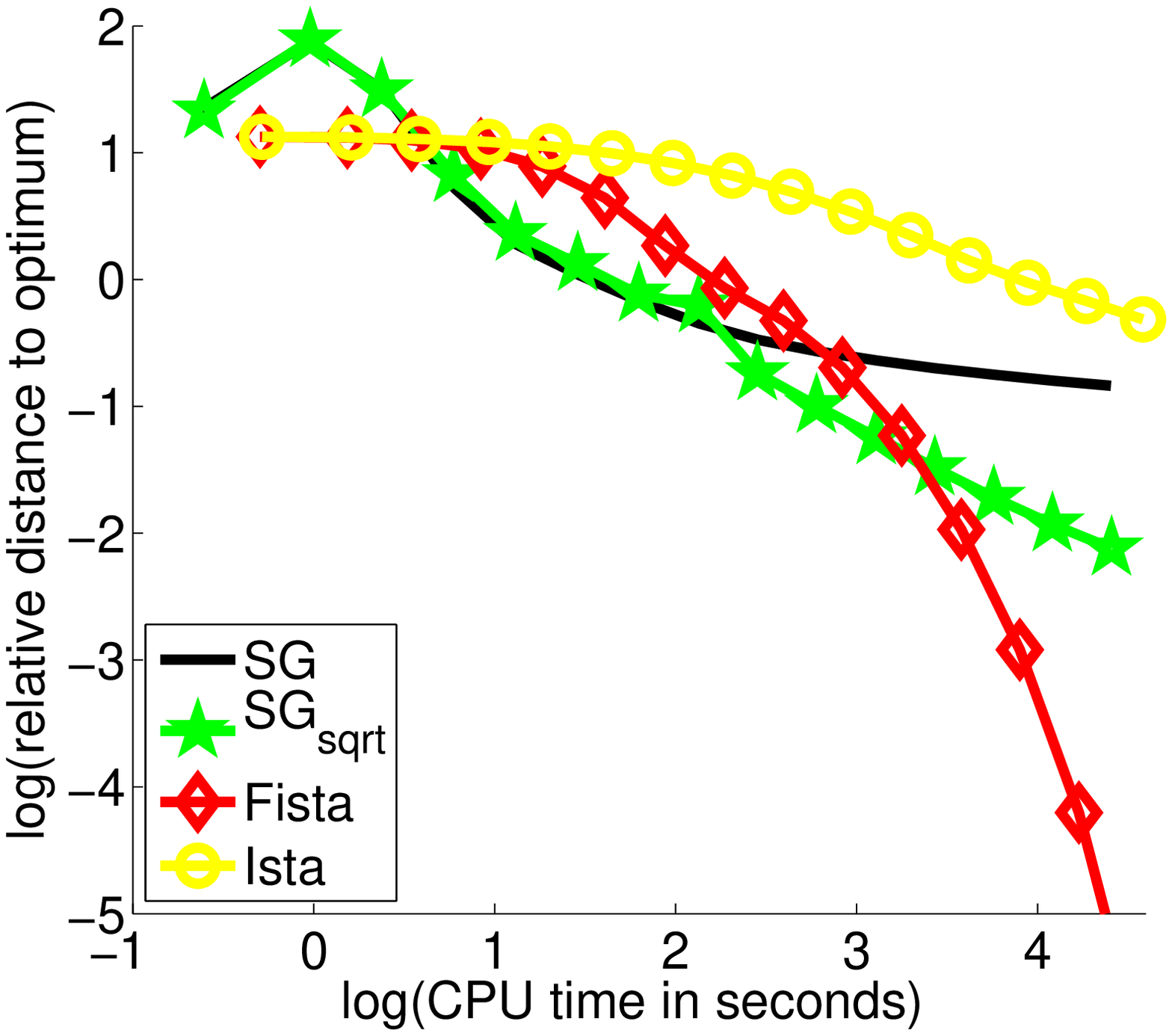}} 
  \subfloat[scale: large, regul.: medium]{\includegraphics[width=0.33\linewidth]{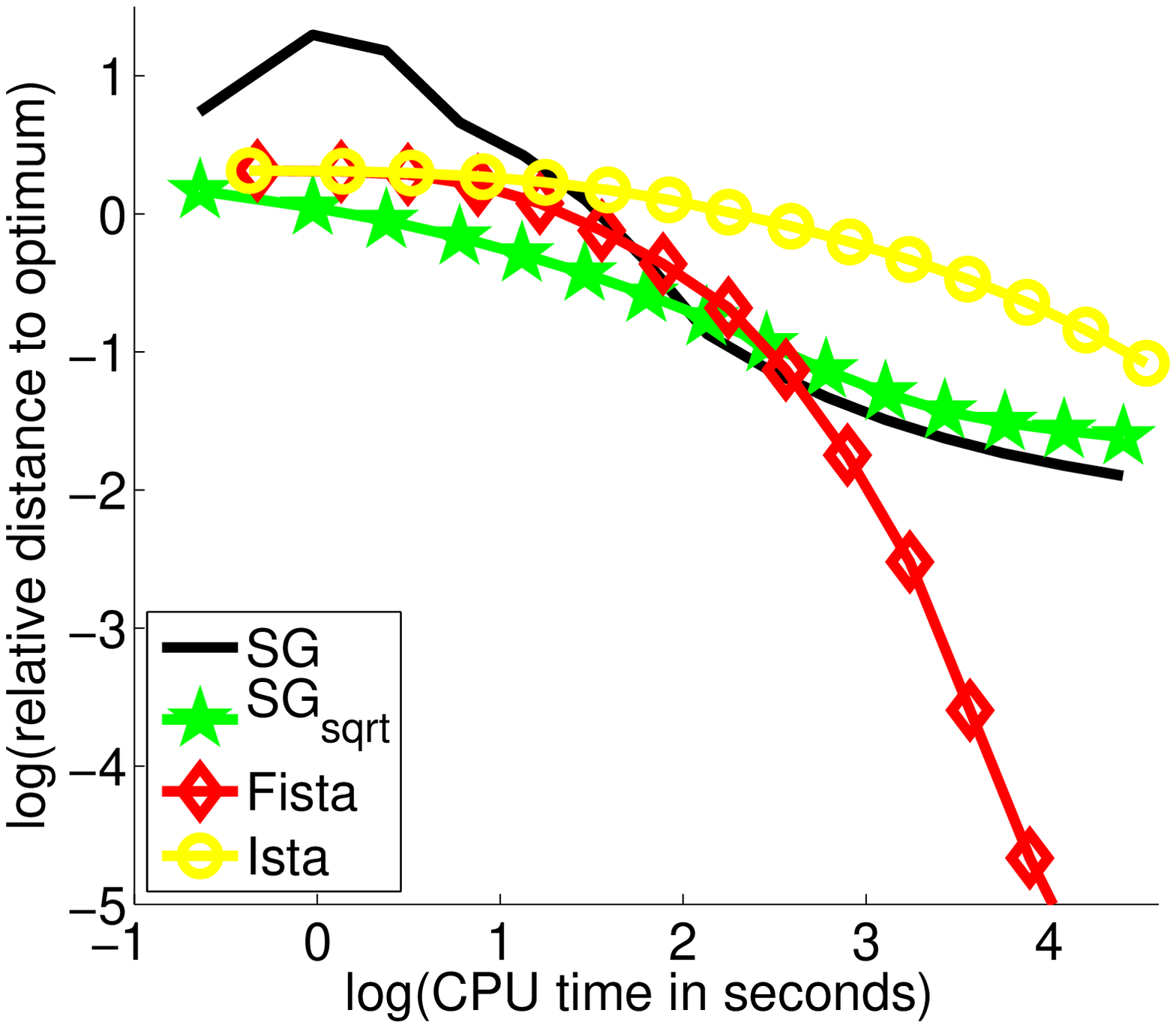}} 
  \subfloat[scale: large, regul.: high]{\includegraphics[width=0.33\linewidth]{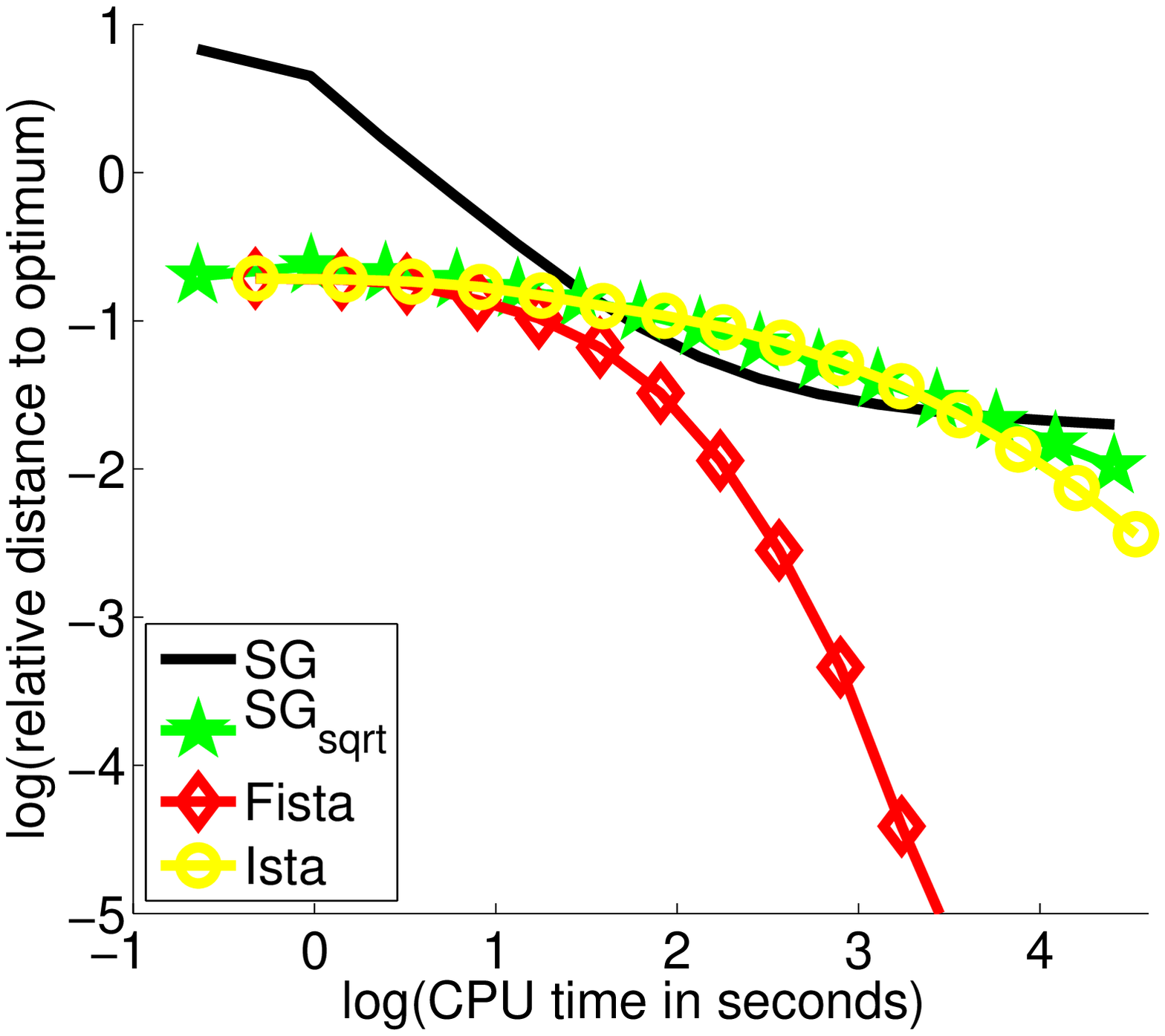}}
  \caption{Benchmark for solving a large-scale multi-class classification problem
   for four optimization methods (see details about the datasets and the methods in the main text).
   Three levels of regularization are considered. The curves represent the relative value of the objective to the optimal value as a function of the computational time in second on a $\log_{10}/\log_{10}$ scale.
   In the highly regularized setting, tuning the step-size for the subgradient turned out to be difficult,
   which explains the behavior of SG in the first iterations.}
  \label{fig:struct_bench_bio}
\end{figure}
\subsubsection{Multi-class classification of cancer diagnosis}
The second benchmark explores a different supervised learning setting, where $f$ is no longer the square loss function.
The goal is to demonstrate that our optimization tools apply in various scenarios, beyond traditional sparse approximation problems.
To this end, we consider a gene expression dataset\footnote{The dataset we use is \textit{14\_Tumors},
which is freely available at \texttt{http://www.gems-system.org/}.}
in the context of cancer diagnosis. More precisely, we
focus on a multi-class classification problem where the number $m$ of
samples to be classified is small compared to the number $p$ of gene
expressions that characterize these samples.  Each atom thus
corresponds to a gene expression across the $m$ samples, whose class labels are
recorded in the vector~$\x$ in~$\R{m}$.

The dataset contains $m=308$ samples, $p=30\,017$
variables and $26$ classes.  In addition, the data exhibit
highly-correlated dictionary elements.  Inspired by~\citet{Kim2009}, we build
the tree-structured set of groups $\G$ using Ward's hierarchical
clustering~\cite{Johnson1967} on the gene expressions.  The norm $\Omega$ built
in this way aims at capturing the hierarchical structure of gene expression
networks~\cite{Kim2009}.

Instead of the square loss function, we consider the multinomial logistic loss
function that is better suited to deal with multi-class classification
problems~\cite[see, e.g.,][]{Hastie2009}.  As a direct consequence,
algorithms whose applicability crucially depends on the choice of the loss
function $f$ are removed from the benchmark.  This is the case with
reweighted-$\ell_2$ schemes that do not have closed-form updates anymore.
Importantly, the choice of the multinomial logistic loss
function leads to an optimization problem over a matrix with dimensions~$p$ times the number
of classes (i.e., a total of  $30\,017\times 26\approx780\,000$ variables).  
Also, due to scalability issues, generic
interior point solvers could not be considered here.  

The results in Figure~\ref{fig:struct_bench_bio} highlight that
the accelerated proximal scheme performs overall better that the two other
methods. Again, it is important to note that both proximal algorithms yield
sparse solutions, which is not the case for SG.

\subsection{Denoising with Tree-Structured Wavelets}

We demonstrate in this section how a tree-structured sparse regularization can
improve classical wavelet representation, and how our method can be used to
efficiently solve the corresponding large-scale optimization problems.  We
consider two wavelet orthonormal bases, Haar and Daubechies3~\cite[see][]{Mallat1999}, and choose a classical quad-tree structure on the
coefficients, which has notably proven to be useful for image compression
problems~\cite{Baraniuk1999}. This experiment follows the approach of
\citet{Zhao2009} who used the same tree-structured regularization in the case of
small one-dimensional signals, and the approach of~\citet{Baraniuk2008} and
\citet{Huang2009} images where images were reconstructed from compressed sensing
measurements with a hierarchical nonconvex penalty.

We compare the performance for image denoising of both nonconvex and convex approaches.
Specifically, we consider the following formulation
$$
\min_{\alphab\in\R{m}} \frac{1}{2} \|\x-\D\alphab\|_2^2+\lambda\psi(\alphab)
=\min_{\alphab\in\R{m}} \frac{1}{2}\|\D^\top\x-\alphab\|_2^2 + \lambda\psi(\alphab),
$$
where $\D$ is one of the orthonormal wavelet basis mentioned above,
$\x$ is the input noisy image, $\D\alphab$ is the estimate of the denoised image, 
and $\psi$ is a sparsity-inducing regularization. Note that in this case, $m=p$.
We first consider classical settings where $\psi$ is either the $\ell_1$-norm---
this leads to the wavelet soft-thresholding method
of \citet{Donoho1995}---
or the $\ell_0$-pseudo-norm, whose solution can be obtained
by hard-thresholding~\cite[see][]{Mallat1999}.
Then, we consider the convex tree-structured
regularization $\Omega$ defined as a sum of $\ell_2$-norms ($\ell_\infty$-norms), 
which we denote by $\Omega_{\ell_2}$ (respectively
$\Omega_{\ell_\infty}$). Since the basis is here orthonormal, solving the corresponding decomposition problems amounts
to computing a single instance of the proximal operator. 
As a result, when $\psi$ is~$\Omega_{\ell_2}$, 
we use Algorithm~\ref{alg:fbcd} and for $\Omega_{\ell_\infty}$, 
Algorithm~\ref{alg:bcd2} is applied.
Finally, we consider the nonconvex tree-structured
regularization used by \citet{Baraniuk2008} denoted here by $\ell_0^{\text{tree}}$,
which we have presented in Eq.~(\ref{eq:nonconvex}); 
the implementation details for $\ell_0^{\text{tree}}$ can be found in
Appendix~\ref{appendix:greedy}.

\begin{table}[h!]
\centering
\begin{tabular}{|c|c|*{5}{|c}|}
\hline
& & \multicolumn{5}{c|}{Haar} \\
\cline{2-7}
&$\sigma$ & $\ell_0\; [0.0012]$ & $\ell_0^\text{tree}\; [0.0098]$ 
& $\ell_1\; [0.0016]$ & $\Omega_{\ell_2}\; [0.0125]$ & $\Omega_{\ell_\infty}\; [0.0221]$ \\
\hline
\multirow{5}{*}{PSNR} 
&$5$   & 34.48 & 34.78 & 35.52 & \textbf{35.89} & 35.79 \\
\cline{2-7}
&$10$  & 29.63 & 30.24 & 30.74 & \textbf{31.40} & 31.23  \\
\cline{2-7}
&$25$  & 24.44 & 25.27 & 25.30 & \textbf{26.41} & 26.14  \\
\cline{2-7}
&$50$  & 21.53 & 22.37 & 20.42 & \textbf{23.41} & 23.05  \\
\cline{2-7}
&$100$ & 19.27 & 20.09 & 19.43 & \textbf{20.97} & 20.58  \\
\hline
\hline
\multirow{5}{*}{IPSNR} 
&$5$ & - & $.30\pm.23$ & $1.04\pm.31$ & $\mathbf{1.41\pm.45}$ & $1.31\pm.41$ \\
\cline{2-7}
&$10$ & - & $.60\pm.24$ & $1.10\pm.22$ & $\mathbf{1.76\pm.26}$ & $1.59\pm.22$ \\
\cline{2-7}
&$25$ & - & $.83\pm.13$ & $.86\pm.35$ & $\mathbf{1.96\pm.22}$ & $1.69\pm.21$\\
\cline{2-7}
&$50$ & - & $.84\pm.18$ & $.46\pm.28$ & $\mathbf{1.87\pm.20}$ & $1.51\pm.20$ \\
\cline{2-7}
&$100$ & - & $.82\pm.14$ & $.15\pm.23$ & $\mathbf{1.69\pm.19}$ & $1.30\pm.19$ \\
\hline
\noalign{\vspace*{0.3cm}} 
\hline
& & \multicolumn{5}{c|}{Daub3} \\
\cline{2-7}
&$\sigma$ & $\ell_0\; [0.0013]$ & $\ell_0^\text{tree}\; [0.0099]$ 
& $\ell_1\; [0.0017]$ & $\Omega_{\ell_2}\; [0.0129]$ & $\Omega_{\ell_\infty}\; [0.0204]$ \\
\hline
\multirow{5}{*}{PSNR} 
&$5$   & 34.64 & 34.95 & 35.74 & \textbf{36.14} & 36.00 \\
\cline{2-7}
&$10$  &  30.03 & 30.63 & 31.10 & \textbf{31.79} & 31.56 \\
\cline{2-7}
&$25$  & 25.04 & 25.84 & 25.76 & \textbf{26.90} & 26.54 \\
\cline{2-7}
&$50$  & 22.09 & 22.90 & 22.42 & \textbf{23.90} & 23.41 \\
\cline{2-7}
&$100$ & 19.56 & 20.45 & 19.67 & \textbf{21.40} & 20.87 \\
\hline
\hline
\multirow{5}{*}{IPSNR} 
&$5$ & - & $.31\pm.21$ & $1.10\pm.23$ & $\mathbf{1.49\pm.34}$ & $1.36\pm.31$ \\
\cline{2-7}
&$10$ & - & $.60\pm.16$ & $1.06\pm.25$& $\mathbf{1.76\pm.19}$ & $1.53\pm.17$ \\
\cline{2-7}
&$25$ & - & $.80\pm.10$ & $.71\pm.28$ & $\mathbf{1.85\pm.17}$ & $1.50\pm.18$ \\
\cline{2-7}
&$50$ & - & $.81\pm.15$ & $.33\pm.24$ & $\mathbf{1.80\pm.11}$ & $1.33\pm.12$ \\
\cline{2-7}
&$100$ & - & $.89\pm.13$ & $0.11\pm.24$ & $\mathbf{1.82\pm.24}$ & $1.30\pm.17$ \\
\hline
\end{tabular}
\caption{Top part of the tables: Average PSNR measured for the denoising of $12$
standard images, when the wavelets are Haar or Daubechies3
wavelets~\cite[see][]{Mallat1999}, 
for two nonconvex approaches ($\ell_0$ and $\ell_0^{\text{tree}}$) and three different convex regularizations---that is, the $\ell_1$-norm, the tree-structured sum of $\ell_2$-norms ($\Omega_{\ell_2}$), 
and the tree-structured sum of $\ell_\infty$-norms ($\Omega_{\ell_\infty}$). 
Best results for each level of noise and each wavelet type are in bold. 
Bottom part of the tables: Average improvement in PSNR with respect to the $\ell_0$ nonconvex method (the standard deviations are computed over the 12 images).
CPU times (in second) averaged over all images and noise realizations are reported in brackets next to the names of the methods they correspond to.}
\label{table:wave}
\end{table}
Compared to \citet{Zhao2009}, the novelty of our approach is essentially to be
able to solve efficiently and exactly large-scale instances of this problem.
We use $12$ classical standard test images,\footnote{These images are used in
classical image denoising benchmarks. See~\citet{Mairal2009b}.} and generate
noisy versions of them corrupted by a white Gaussian noise of variance
$\sigma$. For each image, we test several values of $\lambda =
2^{\frac{i}{4}}\sigma\sqrt{\log{m}}$, with~$i$ taken in a specific
range.\footnote{For the convex formulations, $i$ ranges in
$\{-15,-14,\dots,15\}$, while in the nonconvex case $i$ ranges in
$\{-24,\dots,48\}$.}
We then keep the parameter $\lambda$
giving the best reconstruction error.
The factor $\sigma\sqrt{\log{m}}$ is a classical heuristic for choosing a reasonable regularization parameter~\citep[see][]{Mallat1999}.
We provide reconstruction results in terms of PSNR in
Table~\ref{table:wave}.\footnote{Denoting by \textrm{MSE} the mean-squared-error
for images whose intensities are between $0$ and $255$, the
\textrm{PSNR} is defined as $\textrm{PSNR}=10\log_{10}( 255^2 /
\textrm{MSE} )$ and is measured in dB. A gain of $1$dB reduces
the \textrm{MSE} by approximately $20\%$.}
We report in this table the results when $\Omega$ is chosen to be a sum of
$\ell_2$-norms or $\ell_\infty$-norms with weights~$\weights_g$ all equal to one.
Each experiment was run $5$ times with different noise realizations.
In every setting, we observe that the tree-structured norm significantly outperforms the
$\ell_1$-norm and the nonconvex approaches. 
We also present a visual comparison on
two images on Figure~\ref{fig:wave}, showing that the tree-structured
norm reduces visual artefacts (these artefacts are better seen by zooming on a computer screen).
The wavelet transforms in our experiments are computed with the
matlabPyrTools
software.\footnote{\texttt{http://www.cns.nyu.edu/{\texttildelow}eero/steerpyr/}.}

\begin{figure}[hbtp]
   \centering
   \subfloat[\textsf{Lena}, $\sigma=25$, $\ell_1$]{\includegraphics[width=0.245\linewidth]{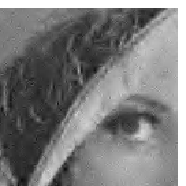}} \hfill
   \subfloat[\textsf{Lena}, $\sigma=25$, $\Omega_{\ell_2}$]{\includegraphics[width=0.245\linewidth]{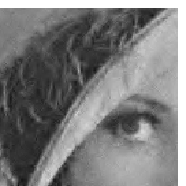}} \hfill
   \subfloat[\textsf{Barb.}, $\sigma=50$, $\ell_1$]{\includegraphics[width=0.245\linewidth]{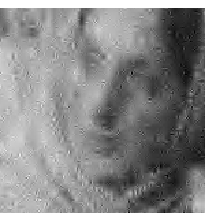}} \hfill
   \subfloat[\textsf{Barb.}, $\sigma=50$, $\Omega_{\ell_2}$]{\includegraphics[width=0.245\linewidth]{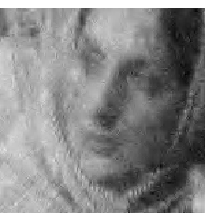}}
   \caption{Visual comparison between the wavelet shrinkage model with the $\ell_1$-norm and the tree-structured model, on cropped versions of the images \textsf{Lena} and \textsf{Barb.}. Haar wavelets are used.}
   \label{fig:wave}
\end{figure}
This experiment does of course not provide state-of-the-art results for image
denoising~\cite[see][and references therein]{Mairal2009b}, but shows that
the tree-structured regularization significantly improves the reconstruction
quality for wavelets. In this experiment the convex setting~$\Omega_{\ell_2}$ and 
$\Omega_{\ell_\infty}$ also outperforms the nonconvex one $\ell_0^{\text{tree}}$.\footnote{
It is worth mentioning that comparing convex and nonconvex approaches for
sparse regularization is a bit difficult. This conclusion holds
for the classical formulation we have used, but might not hold in other
settings such as~\citet{Coifman1995}.}
We also note that the speed of our approach makes it scalable to real-time
applications. Solving the proximal problem for an image with $m=512 \times 512
= 262\,144$ pixels takes approximately $0.013$ seconds on
a single core of a 3.07GHz CPU if $\Omega$ is a sum of ${\ell_2}$-norms, and $0.02$ seconds when it is a sum of $\ell_\infty$-norms.
By contrast, unstructured approaches have a speed-up factor of about 7-8 with respect to the tree-structured methods.

\subsection{Dictionaries of Natural Image Patches}\label{sec:natural_img_patches}
This experiment studies whether a hierarchical structure can help
dictionaries for denoising natural image patches, and in which noise
regime the potential gain is significant.
We aim at reconstructing \emph{corrupted}
patches from a test set, after having learned dictionaries on a training
set of \emph{non-corrupted} patches.
Though not typical
in machine learning, this setting is reasonable in the context of images, where
lots of non-corrupted patches are easily
available.\footnote{Note that we study the ability of the model
to reconstruct independent patches, and additional work is required to apply
our framework to a full image processing task, where patches usually
overlap~\cite{Elad2006,Mairal2009b}. }
\begin{table}[h!]
   \centering {
   \begin{tabular}{|@{\hspace{1mm}}c@{\hspace{1mm}}|*{5}{@{\hspace{1mm}}c@{\hspace{1mm}}|}}
      \hline
      noise & 50 \% & 60 \%  &  70 \% &  80 \% &  90 \%   \tabularnewline
      \hline
      flat & $19.3 \pm 0.1$   & $26.8 \pm 0.1$  & $36.7 \pm 0.1$  & $50.6 \pm 0.0$  & $72.1 \pm 0.0$  \tabularnewline
      \hline
      tree & $18.6 \pm 0.1$   & $25.7 \pm 0.1$  & $35.0 \pm 0.1$  & $48.0 \pm 0.0$  & $65.9 \pm 0.3$  \tabularnewline
      \hline
   \end{tabular} }
   \caption{Quantitative results of the reconstruction task on natural image patches.
   First row: percentage of missing pixels. Second and third rows: mean square error multiplied by $100$, respectively for classical sparse coding, and tree-structured sparse coding.}\label{table:noise1}
\end{table}
\begin{figure}[h!]
   \centering
   \includegraphics[width=0.7\linewidth]{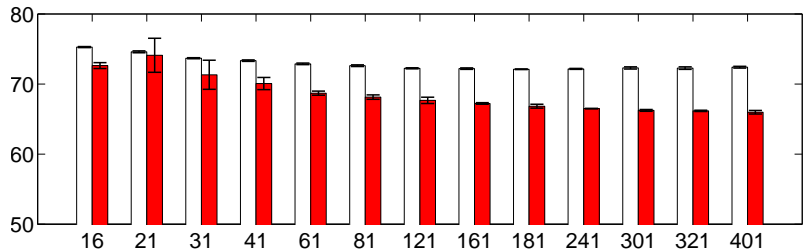}
   \caption{Mean square error multiplied by $100$ obtained with $13$ structures with error bars, sorted by number of dictionary elements from $16$ to $401$. Red plain bars represents the tree-structured dictionaries. White bars correspond to the flat dictionary model containing the same number of dictionary as the tree-structured one. For readability purpose, the $y$-axis of the graph starts at $50$.}\label{fig:errorbars}
\end{figure}

We extracted $100\,000$ patches of size $m=8 \times 8$ pixels
from the Berkeley segmentation database of natural
images \cite{Martin2001}, which contains a high variability of scenes.  We
then split this dataset into a training set~$\X_{tr}$, a validation set
$\X_{val}$, and a test set $\X_{te}$, respectively of size $50\,000$,
$25\,000$, and $25\,000$ patches. All the patches are centered and normalized
to have unit $\ell_2$-norm.

For the first experiment, the dictionary $\D$ is learned on $\X_{tr}$ using the
formulation of Eq.~(\ref{eq:general_formulation}), with $\mu=0$ for~${\mathcal D_{\mu}}$ as defined in Eq.~(\ref{eq:d_mu}).
The validation and test sets are corrupted by removing a certain percentage of
pixels, the task being to reconstruct the missing pixels from the known pixels.
We thus introduce for each element $\x$ of the validation/test set, a vector
$\tildex$, equal to $\x$ for the known pixel values and $0$ otherwise. Similarly, we define~$\tildeD$ as the matrix equal to $\D$, except for the rows
corresponding to missing pixel values, which are set to $0$.  By decomposing~$\tildex$ on~$\tildeD$, we obtain a sparse code $\alphab$, and the estimate of
the reconstructed patch is defined as $\D\alphab$.  Note that this procedure
assumes that we know which pixel is missing and which is not for every element
$\x$.

The parameters of the experiment are the regularization parameter
$\lambda_{tr}$ used during the training step, the regularization parameter
$\lambda_{te}$ used during the validation/test step, and the structure of the
tree.  For every reported result, these parameters were selected by taking the
ones offering the best performance on the \emph{validation} set, before
reporting any result from the \emph{test} set.  The values for the
regularization parameters $\lambda_{tr}, \lambda_{te}$ were selected on a
logarithmic scale $\{2^{-10},2^{-9},\ldots,2^{2}\}$, and then further refined
on a finer logarithmic scale with multiplicative increments of $2^{-1/4}$.  For simplicity, we
chose arbitrarily to use the $\ell_\infty$-norm in the structured norm
$\Omega$, with all the weights equal to one.
We tested $21$ balanced tree structures
of depth $3$ and~$4$, with different \emph{branching factors} $p_1, p_2,\ldots,
p_{d-1}$, where~$d$ is the depth of the tree and $p_k$, $k \in
\{1,\ldots,d-1\}$ is the number of children for the nodes at depth $k$.
The branching factors tested for the trees of depth $3$ where $p_1 \in
\{5,10,20,40,60,80,100\}$, $p_2 \in \{2,3\}$, and for trees of depth $4$, $p_1
\in \{5,10,20,40\}$, $p_2 \in \{2,3\}$ and $p_3 =2$, giving $21$ possible
structures associated with dictionaries with at most $401$ elements.
For each tree structure, we evaluated the performance obtained with the
tree-structured dictionary along with a non-structured dictionary containing
the same number of elements.  These experiments were carried out four times, each
time with a different initialization, and with a different
noise realization.

Quantitative results are reported in Table \ref{table:noise1}. For all fractions of missing pixels considered, the tree-structured dictionary outperforms the
``unstructured one'', and the most significant improvement is obtained in the
noisiest setting. Note that having more dictionary elements is worthwhile when
using the tree structure.
To study the influence of the chosen structure, we report in Figure~\ref{fig:errorbars}
the results obtained with the $13$ tested structures of depth $3$, along with those obtained
with unstructured dictionaries containing the same number of elements,
when $90 \%$ of the pixels are missing.  For each dictionary
size, the tree-structured dictionary significantly outperforms the
unstructured one.
An example of a learned tree-structured dictionary is presented on Figure
\ref{fig:tree2}. Dictionary elements naturally organize in groups of
patches, often with low frequencies near the root of the tree,
and high frequencies near the leaves. 
\begin{figure}
\centering
   \includegraphics[width=0.5\linewidth]{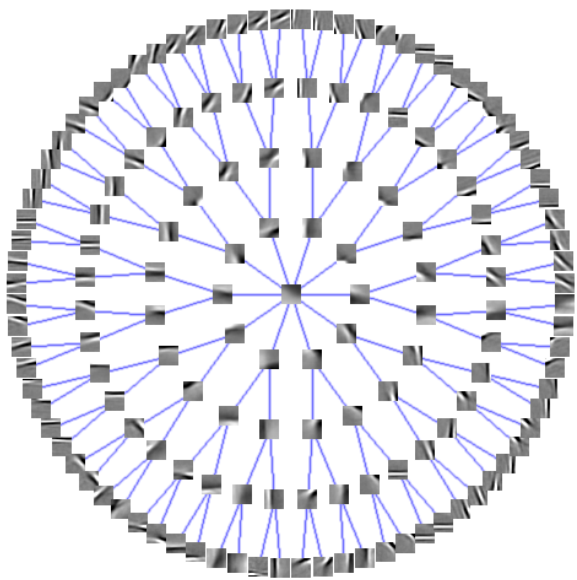}
   \caption{Learned dictionary with a tree structure of depth $5$. The root of the tree is in the middle of the figure.
   The branching factors are $p_1=10$, $p_2=2$, $p_3=2$, $p_4=2$. The dictionary is learned on $50,000$ patches of size $16 \times 16$ pixels.}
   \label{fig:tree2}
\end{figure}

\subsection{Text Documents}\label{sec:exp_txt_documents}
This last experimental section shows that our approach can also be applied to model text corpora.
The goal of probabilistic topic models is to find a low-dimensional representation
of a collection of documents, where
the representation should provide a semantic description of the collection.
Approaching the problem in a parametric Bayesian framework,
latent Dirichlet allocation (LDA) \citet{Blei2003} model documents, represented as vectors of word counts,
as a mixture of a predefined number of \emph{latent topics} that are distributions over a fixed vocabulary. 
LDA is fundamentally a matrix factorization problem: \citet{Buntine2002} shows that LDA can be interpreted as a Dirichlet-multinomial counterpart of factor analysis.
The number of topics is usually small compared to the size of the vocabulary (e.g., 100 against $10\, 000$),
so that the topic proportions of each document provide a compact representation of the corpus.
For instance, these new features can be used to feed a classifier in a subsequent classification task.
We similarly use our dictionary learning approach to find low-dimensional representations of text corpora.

Suppose that the signals
$ \X = [\x^1,\dots,\x^n]$ in $\RR{m}{n}$
are each the \emph{bag-of-word} representation of each of $n$ documents over a vocabulary of $m$ words,
the $k$-th component of $\x^i$ standing for the frequency of the $k$-th word in the document $i$.
If we further assume that the entries of $\D$ and $\A$ are nonnegative,
and that the dictionary elements~$\d^j$ have unit $\ell_1$-norm,
the decomposition~$(\D,\A)$ can be interpreted as the parameters of a topic-mixture model.
The regularization $\Omega$ induces the organization of these topics on a tree, so that,
if a document involves a certain topic, then all ancestral topics in the tree are also present in
the topic decomposition.
Since the hierarchy is shared by all documents, the topics at the top of the tree
participate in every decomposition,
and should therefore gather the lexicon which is common to all documents.
Conversely, the deeper the topics in the tree, the more specific they should be.
An extension of LDA to model topic hierarchies was proposed by \citet{Blei2010}, who introduced
a non-parametric Bayesian prior over trees of topics and modelled documents as
convex combinations of topics selected along a path in the hierarchy.
We plan to compare our approach with this model in future work.
\begin{figure}[h!]
\centering
\includegraphics[width=0.55\linewidth]{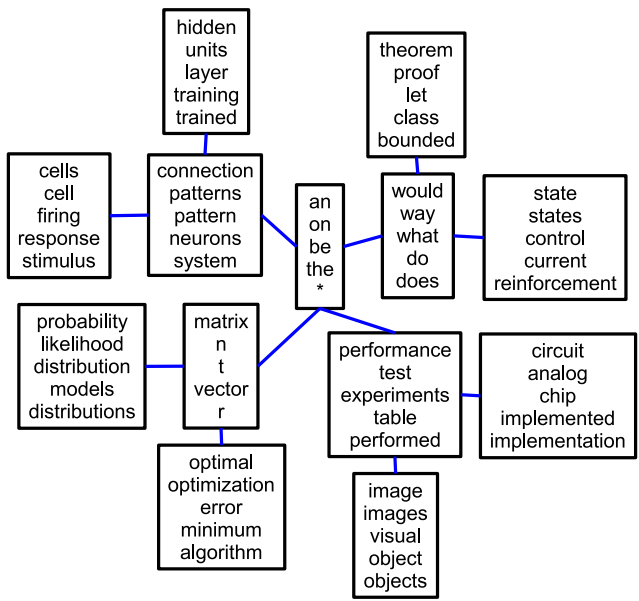}
\caption{Example of a topic hierarchy estimated from 1714 NIPS proceedings papers (from 1988 through 1999).
Each node corresponds to a topic whose 5 most important words are displayed.
Single characters such as $n, t, r$ are part of the vocabulary and often appear in NIPS papers, and their place in the hierarchy is semantically relevant to children topics.}
\label{fig:topics_nips}
\end{figure}

\paragraph{Visualization of NIPS proceedings.}
We qualitatively illustrate our approach on the NIPS proceedings
from 1988 through 1999 \cite{Griffiths2004}.
After removing words appearing fewer than 10 times, the dataset is composed of 1714 articles, with
a vocabulary of 8274 words.
As explained above, we consider $\mathcal{D}_1^+$ and take $\mathcal{A}$ to be $\RR{p}{n}_+$.
Figure~\ref{fig:topics_nips} displays an example of a learned dictionary with 13 topics,
obtained by using the $\ell_\infty$-norm in $\Omega$ and selecting manually $\lambda\!=\!2^{-15}$.
As expected and similarly to~\citet{Blei2010}, we capture the stopwords at the root of the tree, and topics reflecting the different subdomains of the conference
such as neurosciences, optimization or learning theory.
\begin{figure}
\centering
\includegraphics[width=0.6\linewidth]{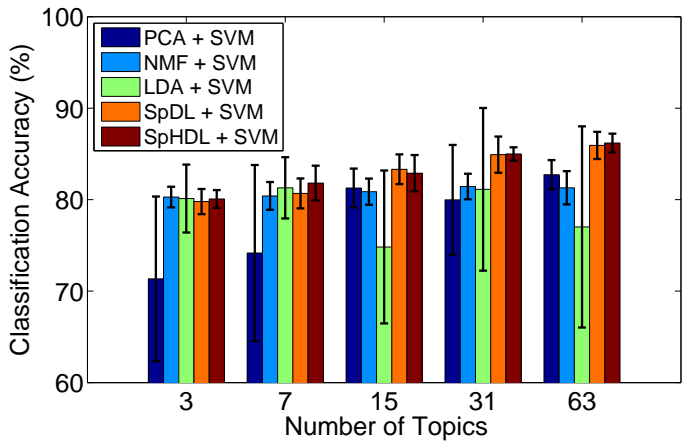}
\caption{Binary classification of two newsgroups: classification accuracy for different dimensionality reduction techniques coupled with a linear SVM classifier.
The bars and the errors are respectively the mean and the standard deviation, based on 10 random splits of the dataset. Best seen in color.}
\label{fig:cmp_LDA}
\end{figure}
\paragraph{Posting classification.}
We now consider a binary classification task of $n$ postings from the 20 Newsgroups data set.\footnote{Available at \texttt{http://people.csail.mit.edu/jrennie/20Newsgroups/}.}
We learn to discriminate between the postings from the two newsgroups \textit{alt.atheism} and \textit{talk.religion.misc},
following the setting of \citet{Lacoste-Julien2008} and \citet{Zhu2009}.
After removing words appearing fewer than 10 times and standard stopwords,
these postings form a data set of 1425 documents over a vocabulary of 13312 words.
We compare different dimensionality reduction techniques that we use to feed a linear SVM classifier, i.e.,
we consider
(i) LDA, with the code from \citet{Blei2003},
(ii) principal component analysis (PCA),
(iii) nonnegative matrix factorization (NMF),
(iv) standard sparse dictionary learning (denoted by SpDL) and
(v) our sparse hierarchical approach (denoted by SpHDL).
Both SpDL and SpHDL are optimized over $\mathcal{D}_1^+$ and $\mathcal{A}\!=\!\RR{p}{n}_+$,
with the weights $\weights_g$ equal to 1.
We proceed as follows: given a random split into a training/test set of $1\,000/425$ postings,
and given a number of topics $p$ (also the number of components for PCA, NMF, SpDL and SpHDL),
we train an SVM classifier based on the low-dimensional representation of the postings.
This is performed on a training set of $1\,000$ postings, where the parameters,
$\lambda \! \in \! \{2^{-26},\dots,2^{-5}\}$
and/or
$C_{\text{svm}} \!\in\! \{4^{-3},\dots,4^1\}$
are selected by 5-fold cross-validation.
We report in Figure~\ref{fig:cmp_LDA} the average classification scores on the test set of 425 postings, based on 10 random splits, for different number of topics.
Unlike the experiment on image patches, we consider only complete binary trees with depths in $\{1,\dots,5\}$.
The results from Figure~\ref{fig:cmp_LDA} show that SpDL and SpHDL perform better than the other dimensionality reduction techniques on this task.
As a baseline, the SVM classifier applied directly to the raw data (the 13312 words) obtains a score of $90.9\! \pm \! 1.1$, which is better
than all the tested methods, but without dimensionality reduction \cite[as already reported by][]{Blei2003}.
Moreover, the error bars indicate that, though nonconvex, SpDL and SpHDL do not seem to suffer much from instability issues.
Even if SpDL and SpHDL perform similarly, SpHDL has the advantage to provide a more interpretable topic mixture in terms of hierarchy, which standard unstructured sparse coding does not.
\section{Discussion}\label{sec:ccl}
We have applied hierarchical sparse coding in various settings, with fixed/learned dictionaries,
and based on different types of data, namely, natural images and text documents.
A line of research to pursue is to develop other optimization tools
for structured norms with general overlapping groups. For instance, 
\citet{Mairal2010a} have used network flow optimization techniques for that purpose,
and \citet{Bach2010a} submodular function optimization. 
This framework can also be used in the context of hierarchical kernel learning~\cite{Bach2008},
where we believe that our method can be more efficient than existing ones.

This work establishes a connection between dictionary learning and
probabilistic topic models, which should prove fruitful as the two lines of work have
focused on different aspects of the same unsupervised learning problem: Our
approach is based on convex optimization tools, and provides experimentally
more stable data representations. Moreover, it can be 
easily extended with the
same tools to other types of structures corresponding to other
norms~\cite{Jenatton2009,Jacob2009}.  
It should be noted, however, that, unlike some Bayesian methods, dictionary learning by itself does not provide mechanisms for the automatic selection of model hyper-parameters (such as
the dictionary size or the topology of the tree).
An interesting common line
of research to pursue could be the supervised design of dictionaries, which has been
proved useful in the two frameworks~\cite{Mairal2009a,Bradley2009,Blei2008}.
\acks{This paper was partially supported by grants from the
Agence Nationale de la Recherche (MGA Project) and
from the European Research Council (SIERRA Project 239993).
The authors would like to thank Jean Ponce for interesting discussions and
suggestions for improving this manuscript.  They also would like to thank
Volkan Cevher for pointing out links between our approach and nonconvex
tree-structured regularization and for insightful discussions.
Finally, we thank the reviewers for their constructive and helpful comments.}
\appendix

\section{Links with Tree-Structured Nonconvex Regularization} \label{appendix:greedy}

We present in this section an algorithm introduced by \citet{Donoho1997} in the more general context of approximation from dyadic partitions~\citep[see Section 6 in][]{Donoho1997}.
This algorithm solves the following problem
\begin{equation}
   \min_{\v \in \R{p}} \frac{1}{2}\|\u-\v\|_2^2 + \lambda \sum_{g \in \G} \delta^g(\v),\label{eq:proxgreedy}
\end{equation}
where the $\u$ in $\R{p}$ is given, $\lambda$ is a regularization parameter, $\G$ is a set of
tree-structured groups in the sense of definition~\ref{def:tree_struct},
and the functions $\delta^g$ are defined as in Eq.~(\ref{eq:nonconvex})---that
is, $\delta^g(\v)=1$ if there exists $j$ in $g$ such that $\v_j \neq 0$, and $0$
otherwise. This problem can be viewed as a proximal operator for the nonconvex 
regularization $\sum_{g \in \G} \delta^g(\v)$. As we will show, it can be
solved efficiently, and in fact it can be used to obtain approximate
solutions of the nonconvex problem presented in Eq.~(\ref{eq:ancestor_cond}),
or to solve tree-structured wavelet decompositions as done by~\citet{Baraniuk2008}.

We now briefly show how to derive the dynamic programming approach introduced
by \citet{Donoho1997}. Given a group $g$ in~$\G$, we use the same notations
$\text{root}(g)$ and $\text{children(g)}$ introduced in
Section~\ref{sec:efficiency}. It is relatively easy to show that finding
a solution of Eq.~(\ref{eq:proxgreedy}) amounts to finding the support
$S \subseteq \{1,\ldots,p\}$ of its solution and that the problem can be equivalently rewritten
\begin{equation}
   \min_{S \subseteq \{1,\ldots,p\}} -\frac{1}{2}\|\u_S\|_2^2 + \lambda \sum_{g \in \G} \delta^g(S), \label{eq:proxgreedy2}
\end{equation}
with the abusive notation $\delta^g(S) = 1$ if $g \cap S \neq \emptyset$ and $0$ otherwise.
We now introduce the quantity
\begin{displaymath}
\psi_g(S) \defin \begin{cases}
  0 & ~\text{if}~ g \cap S = \emptyset \\
  -\frac{1}{2}\|\u_{\text{root}(g)}\|_2^2 + \lambda + \sum_{h \in \text{children}(g)} \psi_h(S) & ~\text{otherwise}.
\end{cases}
\end{displaymath}
After a few computations, solving
Eq.~(\ref{eq:proxgreedy2}) can be shown to be equivalent to minimizing $\psi_{g_0}(S)$ where
$g_0$ is the root of the tree. It is then easy to prove that for any group $g$ in $\G$, we have
\begin{displaymath}
   \min_{S \subseteq \{1,\ldots,p\}} \psi_g(S) = \min\Big(0,  -\frac{1}{2}\|\u_{\text{root}(g)}\|_2^2 + \lambda + \sum_{h \in \text{children}(g)}\min_{S' \subseteq \{1,\ldots,p\}} \psi_h(S') \Big),
\end{displaymath}
which leads to the following dynamic programming approach presented in Algorithm~\ref{alg:proxgreedy}.
\begin{algorithm}[h!]
\caption{Computation of the Proximal Operator for the Nonconvex Approach}\label{alg:proxgreedy}
\begin{algorithmic}
\STATE Inputs: $\u \in \R{p}$, a tree-structured set of groups $\G$ and $g_0$ (root of the tree).
\STATE Outputs: $\v$ (primal solution).
\STATE Initialization: $\v \leftarrow \u$.
\STATE Call \texttt{recursiveThresholding}($g_0$).
\end{algorithmic}
{\bf Procedure} \texttt{recursiveThresholding}($g$)
\begin{algorithmic}[1]
   \STATE $\eta \leftarrow \min\Big(0,-\frac{1}{2}\|\u_{\text{root}(g)}\|_2^2 + \lambda + \sum_{h \in \text{children}(g)} \text{\texttt{recursiveThresholding}}(h)\Big)$.
   \IF{ $\eta = 0$ }
        \STATE $\v_g \leftarrow 0$.
   \ENDIF
   \RETURN $\eta$.
\end{algorithmic}
\end{algorithm}
This algorithm shares several conceptual links with Algorithm~\ref{alg:bcd2}
and~\ref{alg:fbcd}.  It traverses the tree in the same order, has a complexity
in $O(p)$, and it can be shown that the whole procedure actually performs a
sequence of thresholding operations on the variable $\v$.

\section{Proofs}\label{sec:proofs}
We gather here the proofs of the technical results of the paper. 
\subsection{Proof of Lemma~\ref{lem:dual}}
\begin{proof}
The proof relies on tools from conic duality \cite{Boyd2004}.
Let us introduce the cone 
$\mathcal{C}\defin\{(\v,z) \in \R{p+1};\ \Norm{\v}  \leq z \}$
and its dual counterpart
$\mathcal{C^\ast} \defin\{(\xib,\tau) \in \R{p+1};\ \DualNorm{\xib}  \leq \tau \}$.
These cones induce generalized inequalities for which Lagrangian duality also applies.
We refer the interested readers to~\citet{Boyd2004} for further details.

We can rewrite problem~(\ref{eq:prox_problem}) as
$$
\min_{\v \in \R{p}, \z \in \R{|\G|}}
\frac{1}{2} \NormDeux{\u-\v}^2 + \lambda  \sum_{g \in \G} \weights_g z_g,\
\text{such that}\ (\v_{\gi},z_g) \in \mathcal{C},\ \forall g\in\G,
$$
by introducing the primal variables $\z = (z_g)_{g \in \G} \in \R{|\G|}$,
with the additional $|\G|$ conic constraints $(\v_{\gi},z_g) \in \mathcal{C}$, for $g\in\G$.

This primal problem is convex and satisfies Slater's conditions for generalized
conic inequalities (i.e., existence of a feasible point in the interior of the domain), which implies that strong duality holds
\cite{Boyd2004}.  We now consider the Lagrangian~$\mathcal{L}$ defined as
$$
\mathcal{L}(\v, \z, \taub, \xib)
=
\frac{1}{2} \NormDeux{\u-\v}^2
+ \lambda  \sum_{g \in \G} \weights_g z_g
- \sum_{g\in\G} \binom{z_g}{\v_{\gi}}^\top \binom{\tau_g}{\xib^g},
$$
with the dual variables $\taub =(\tau_g)_{g \in \G}$ in $\R{|\G|}$, and $\xib=(\xib^g)_{g \in \G}$ in
$\RR{p}{|\G|}$,
such that for all $g\in\G$,
$\xib^g_j = 0 \, \mbox{ if } \, j \notin g$
and
$ (\xib^g,\tau_g) \in \mathcal{C^\ast}$.

The dual function is obtained by minimizing out the primal variables.
To this end, we take the derivatives of $\mathcal{L}$ with respect to the primal variables~$\v$ and~$\z$ and set them to zero,
which leads to
$$
\v - \u - \sum_{ g \in \G } \xib^g  =  0\quad\mathrm{and}\quad
\forall g \in \G,\ \lambda \weights_g - \tau_g = 0.
$$
After simplifying the Lagrangian and flipping (without loss of generality) the sign of $\xib$,
we obtain the dual problem in Eq.~(\ref{eq:dual_problem}).
We derive the optimality conditions from the Karush\textendash Kuhn\textendash Tucker conditions
for generalized conic inequalities \cite{Boyd2004}.
We have that $\{\v, \z, \taub,\xib\}$ are optimal if and only if
\begin{eqnarray*}
&\forall g \in \G,  z_g \tau_g - \v_{\gi}^\top \xib^g & = 0, \quad \text{(Complementary slackness)}\\                                    
\forall g \in \G,  (\v_{\gi},z_g) \in \mathcal{C}, &\forall g \in \G, \lambda \weights_g -  \tau_g & = 0, \\
\forall g \in \G,  (\xib^g,\tau_g) \in \mathcal{C^\ast}, &\v - \u + \sum_{ g \in \G} \xib^g & =  0.
\end{eqnarray*}
Combining the complementary slackness with the definition of the dual norm,
we have
$$
\forall g \in \G,\  z_g \tau_g = \v_{\gi}^\top \xib^g \leq \Norm{\v_{\gi}} \DualNorm{\xib^g}.
$$
Furthermore, using the fact that
$\forall g \in \G,\  (\v_{\gi},z_g) \in \mathcal{C}$ and
$(\xib^g,\tau_g)=(\xib^g,\lambda \weights_g) \in \mathcal{C^\ast}$,
we obtain the following chain of inequalities
$$
\forall g \in \G,\  \lambda z_g \weights_g  = \v_{\gi}^\top \xib^g
\leq \Norm{\v_{\gi}} \DualNorm{\xib^g}
                                     \leq z_g \DualNorm{\xib^g}
                                     \leq \lambda z_g \weights_g,
$$
for which equality must hold.
In particular, we have
$\v_{\gi}^\top \xib^g = \Norm{\v_{\gi}} \DualNorm{\xib^g}$ and
$z_g \DualNorm{\xib^g} = \lambda z_g \weights_g$.
If $\v_{\gi}\neq 0$, then $z_g$ cannot be equal to zero, which implies in turn that $\DualNorm{\xib^g} = \lambda \weights_g$.
Eventually, applying Lemma~\ref{lem:opt_cond_proj} gives the advertised optimality conditions.

Conversely, starting from the optimality conditions of Lemma~\ref{lem:dual}, 
and making use again of Lemma~\ref{lem:opt_cond_proj},
we can derive the Karush\textendash Kuhn\textendash Tucker conditions displayed above.
More precisely, we define for all $g\in\G$,
$$
\tau_g \defin \lambda \weights_g\quad \text{and}\quad z_g \defin \Norm{\v_{\gi}}.
$$
The only condition that needs to be discussed is the complementary slackness condition.
If $\v_{\gi} = 0$, then it is easily satisfied.
Otherwise, combining the definitions of $\tau_g,\ z_g$ and the fact that
$$
\v_{\gi}^\top \xib^g = \Norm{\v_{\gi}} \DualNorm{\xib^g}\ \text{and}\ \DualNorm{\xib^g}=\lambda \weights_g,
$$
we end up with the desired complementary slackness.
\end{proof}
\subsection{Optimality condition for the projection on the dual ball}
\begin{lemma}[Projection on the dual ball]\label{lem:opt_cond_proj}~\newline
Let $\w \in \R{p}$ and $t>0$. We have $\kappab = \Pi_{\|.\|_\ast \leq t}(\w)$ if and only if
$$
\begin{cases}
\text{if}\ \DualNorm{\w} \leq t,& \kappab=\w, \\ 	
\text{otherwise,}&
\DualNorm{\kappab} = t\
\mbox{ and }\
\kappab^\top (\w - \kappab) = \DualNorm{\kappab} \Norm{\w - \kappab}.
\end{cases}
$$
\end{lemma}
\begin{proof}
When the vector $\w$ is already in the ball of $\DualNorm{.}$
with radius $t$, i.e., $\DualNorm{\w}\leq t$, the situation is simple,
since the projection $\Pi_{\|.\|_\ast \leq t}(\w)$ obviously gives $\w$ itself.
On the other hand, a necessary and sufficient optimality condition for having
$\kappab=\Pi_{\|.\|_\ast \leq t}(\w)=\argmin_{\DualNorm{\y} \leq
t}\NormDeux{\w-\y}$ 
is that the residual $\w-\kappab$ lies in the normal cone
of the constraint set \cite{Borwein2006}, that is, for all $\y$ such that
$\DualNorm{\y}\! \leq t$, $(\w-\kappab)^\top \!(\y-\kappab)\! \leq 0$.  The
displayed result then follows from the definition of the dual norm, namely
$\DualNorm{\kappab} \!= \! \max_{\Norm{\z}\leq1}\z^\top\kappab$.
\end{proof}
\subsection{Proof of Lemma~\ref{lem:proj_nested_gr}}
\begin{proof}
First, notice that the conclusion $\xib^h = \Pi_{\|.\|_\ast \leq \lambda \weights_h}(\v_{\hi} +\xib^h)$
simply comes from the definition of $\xib^h$ and $\v$, 
along with the fact that $\xib^g = \xib_{\hi}^g$ since $g \subseteq h$.
We now examine $\xib^g$.

The proof mostly relies on the optimality conditions characterizing the projection onto a ball of the dual norm~$\DualNorm{\cdot}$.
Precisely, by Lemma~\ref{lem:opt_cond_proj}, we need to show that either
$$
\xib^g=\u_{\gi}-\xib_{\gi}^h,\ \text{if}\ \DualNorm{\u_{\gi}-\xib_{\gi}^h} \leq t_g,
$$
or
$$
\DualNorm{\xib^g} = t_g\
	\mbox{ and }\
        \xib^{g\top} (\u_{\gi} - \xib_{\gi}^h - \xib^g)
        = \DualNorm{\xib^g} \Norm{\u_{\gi} - \xib_{\gi}^h - \xib^g}.
$$
Note that the feasibility of $\xib^g$, i.e., $\DualNorm{\xib^g}\leq t_g$, 
holds by definition of $\kappa^g$.

Let us first assume that $\DualNorm{\xib^g} < t_g$.
We necessarily have that $\u_{\gi}$ also lies in the interior of the ball of $\DualNorm{.}$ with radius $t_g$,
and it holds that $\xib^g = \u_{\gi}$.
Since $g \subseteq h$, we have that the vector $\u_{\hi}-\xib^g=\u_{\hi}-\u_{\gi}$ has only zero entries on $g$.
As a result, $\xib^h_g=0$ (or equivalently, $\xib^h_\gi=0$) and we obtain
$$
\xib^g=\u_{\gi}=\u_{\gi}-\xib_{\gi}^h,
$$
which is the desired conclusion.
From now on, we assume that $\DualNorm{\xib^g} = t_g$. It then remains to show that
$$
\xib^{g\top} (\u_{\gi} - \xib_{\gi}^h - \xib^g)
        = \DualNorm{\xib^g} \Norm{\u_{\gi} - \xib_{\gi}^h - \xib^g}.
$$
We now distinguish two cases, according to the norm used.
\vspace*{0.25cm}

 \underline{$\ell_2$-norm:} As a consequence of Lemma~\ref{lem:opt_cond_proj},
 the optimality condition reduces to the conditions for equality
 in the Cauchy-Schwartz inequality, i.e.,
 when the vectors have same signs and are linearly dependent.
 Applying these conditions to individual projections we get that there exists $\rho_g,\rho_h > 0$ such that
 \begin{equation}\label{eq:proof_proj_L2norm}
  \rho_g \xib^g = \u_{\gi} -\xib^g\quad \mathrm{and}\quad 
\rho_h \xib^h = \u_{\hi}-\xib^g -\xib^h.
 \end{equation}
Note that the case $\rho_h = 0$ leads to $\u_{\hi}-\xib^g-\xib^h=0$, and therefore $\u_{\gi}-\xib^g-\xib_{\gi}^h=0$ since $g \subseteq h$,
which directly yields the result. The case $\rho_g=0$ implies $\u_{\gi}-\xib^g = 0$ and therefore $\xib_{\gi}^h=0$,
yielding the result as well.
Now, we can therefore assume $\rho_h > 0$ and $\rho_g > 0$.
From the first equality of~(\ref{eq:proof_proj_L2norm}), we have $\xib^g=\xib^g_\gi$ since $(\rho_g+1) \xib^g = \u_{\gi}$.
Further using the fact that $g \subseteq h$ in the second equality of ~(\ref{eq:proof_proj_L2norm}),
we obtain 
$$
(\rho_h+1) \xib^h_\gi=\u_\gi - \xib^g=\rho_g \xib^g.
$$
This implies that $\u_{\gi}-\xib^g-\xib_{\gi}^h=\rho_g \xib^g - \frac{\rho_g}{\rho_h+1} \xib^g$, 
which eventually leads to 
$$\xib^g = \frac{\rho_h+1}{\rho_g \rho_h} (\u_{\gi}-\xib^g-\xib_{\gi}^h).$$
The desired conclusion follows
$\xib^{g\top} (\u_{\gi} - \xib^g -\xib_{\gi}^h) = \NormDeux{\xib^g}\NormDeux{\u_{\gi}- \xib^g-\xib_{\gi}^h}.$

\vspace*{0.25cm}

 \underline{$\ell_\infty$-norm:} In this case, the optimality corresponds to the conditions for equality
 in the $\ell_\infty$-$\ell_1$ H\"older inequality.
 Specifically, $\xib^g = \Pi_{\|.\|_\ast \leq t_g}(\u_{\gi})$ holds if and only if
 for all $\xib^g_j \neq 0, j\in g$, we have
 $$
 \u_j-\xib^g_j = \NormInf{\u_{\gi}-\xib^g}  \sign(\xib^g_j).
 $$
 Looking at the same condition for $\xib^h$, we have that
 $\xib^h = \Pi_{\|.\|_\ast \leq t_h}\big(\u_{\hi}-\xib_g\big)$ holds if and only if
 for all $\xib^h_j \neq 0, j\in h$, we have
 $$
 \u_j -\xib^g_j-\xib^h_j = \NormInf{\u_{\hi} -\xib^g-\xib^h}  \sign(\xib^h_j).
 $$
 From those relationships we notably deduce that for all $j\in g$ such that $\xib^g_j \neq 0$,
 $\sign(\xib^g_j) =\sign(\u_j) =\sign(\xib^h_j) =\sign(\u_j-\xib^g_j) = \sign(\u_j-\xib^g_j-\xib^h_j)
 $.
 Let $j \in g$ such that $\xib^g_j \neq 0$. At this point, using the equalities we have just presented,
 $$
 | \u_j -\xib^g_j-\xib^h_j | = \left\{ \begin{array}{ll}
    \NormInf{\u_{\gi} -\xib^g} & \text{if}~ \xib^h_j=0 \\
    \NormInf{\u_{\hi} -\xib^g-\xib^h} & \text{if}~ \xib^h_j \neq 0.
 \end{array} \right.
 $$
 Since $\NormInf{\u_{\gi} -\xib^g} \geq \NormInf{\u_{\gi} -\xib^g-\xib_{\gi}^h}$ (which can be shown using the sign equalities above), and
 $\NormInf{\u_{\hi} -\xib^g-\xib^h} \geq \NormInf{\u_{\gi} -\xib^g-\xib_{\gi}^h}$ (since $g \subseteq h$), we have
 $$
 \NormInf{\u_{\gi} -\xib^g-\xib_{\gi}^h} \geq |\u_j -\xib^g_j-\xib^h_j| \geq \NormInf{\u_{\gi} -\xib^g-\xib_{\gi}^h},
 $$
 and therefore for all $\xib^{g}_j \neq 0$, $j \in g$, we have 
 $
 \u_j -\xib^g_j-\xib^h_j = \NormInf{\u_{\gi} -\xib^g-\xib_{\gi}^h} \sign(\xib^g_j),
 $
which yields the result.
\end{proof}
\subsection{Proof of Lemma~\ref{lemma:l2}}
\begin{proof}
   Notice first that the procedure \texttt{computeSqNorm} is called exactly once for each
   group~$g$ in~$\G$, computing a set of scalars $(\rho_g)_{g \in \G}$ in an order which is
   compatible with the convergence in one pass of
   Algorithm~\ref{alg:bcd}---that is, the children of a node are processed
   prior to the node itself.
   Following such an order, the update of the group $g$ in the original
   Algorithm~\ref{alg:bcd} computes the variable $\xib^g$ which updates
   implicitly the primal variable as follows
   \begin{displaymath}
      \v_{\gi} \leftarrow \big(1-\frac{\lambda \weights_g}{\|\v_{\gi}\|_2}\big)_+ \, \v_{\gi}.
   \end{displaymath}
   It is now possible to show by induction that for all
   group $g$ in $\G$, after a call to the procedure \texttt{computeSqNorm}($g$), the
   auxiliary variable $\eta_g$ takes the value $\|\v_{\gi}\|_2^2$ where $\v$ has the same value as during
   the iteration $g$ of Algorithm~\ref{alg:bcd}.
   Therefore, after calling the procedure \texttt{computeSqNorm}($g_0$), where $g_0$ is
   the root of the tree, the values $\rho_g$ correspond to the successive
   scaling factors of the variable $\v_{\gi}$ obtained during the execution of
   Algorithm~\ref{alg:bcd}.
   After having computed all the scaling factors $\rho_g$, $g\in \G$, the
   procedure \texttt{recursiveScaling} ensures that each variable~$j$ in
   $\{1,\ldots,p\}$ is scaled by the product of all the $\rho_h$, where $h$ is
   an ancestor of the variable $j$.

   The complexity of the algorithm is easy to characterize: Each procedure
\texttt{computeSqNorm} and \texttt{recursiveScaling} is called $p$ times, each
call for a group $g$ has a constant number of operations plus as many
operations as the number of children of $p$.  Since each child can be called
at most one time, the total number of operation of the algorithm is $O(p)$.
\end{proof}
\subsection{Sign conservation by projection}
The next lemma specifies a property for projections when
$\|.\|$ is further assumed to be a $\ell_q$-norm (with $q \geq 1$).
We recall that in that case, $\|.\|_\ast$ is simply the $\ell_{q'}$-norm, with $q' = (1-1/q)^{-1}$.
\begin{lemma}[Projection on the dual ball and sign property]\label{lem:proj_same_sign}~\newline
Let $\w \in \R{p}$ and $t>0$.
Let us assume that $\|.\|$ is a $\ell_q$-norm (with $q \geq 1$). 
Consider also a diagonal matrix $\Sb\in\RR{p}{p}$ whose diagonal entries are in $\{-1,1\}$.
We have
$
 \Pi_{\|.\|_\ast \leq t}(\w) = \Sb \Pi_{\|.\|_\ast \leq t}(\Sb \w).
$
\end{lemma}
\begin{proof}
Let us consider $\kappab=\Pi_{\|.\|_\ast \leq t}(\w)$.
Using essentially the same argument as in the proof of Lemma~\ref{lem:opt_cond_proj}, we have
for all $\y$ such that $\|\y\|_{q'}\! \leq t$, $(\w-\kappab)^\top \!(\y-\kappab)\! \leq 0$.
Noticing that $\Sb^\top \Sb = \I$ and $\|\y\|_{q'}=\|\Sb\y\|_{q'}$, we further obtain
$(\Sb\w-\Sb\kappab)^\top \!(\y'-\Sb\kappab)\! \leq 0$ for all $\y'$ with $\|\y'\|_{q'}\! \leq t$.
This implies in turn that 
$\Sb \Pi_{\|.\|_\ast \leq t}(\w) = \Pi_{\|.\|_\ast \leq t}(\Sb \w)$, which is equivalent to the advertised conclusion.
\end{proof}
Based on this lemma, note that we can assume without loss of generality that the vector we want to project (in this case, $\w$) has only nonnegative entries.
Indeed, it is sufficient to store beforehand the signs of that vector,
compute the projection of the vector with nonnegative entries,
and assign the stored signs to the result of the projection.
\subsection{Non-negativity constraint for the proximal operator}\label{sec:nonneg_prox}
The next lemma shows how we can easily add a non-negativity constraint on the proximal operator when the norm $\Omega$ is 
\textit{absolute}~\citep[Definition 1.2]{Stewart1990}, that is, 
a norm for which the relation 
$\Omega(\u) \leq \Omega(\w)$ 
holds for any two vectors $\w$ and $\u \in \R{p}$ such that $|\u_j|\leq|\w_j|$ for all $j$.
\begin{lemma}[Non-negativity constraint for the proximal operator]\label{lem:nonneg_prox}~\newline
Let $\kappab \in \R{p}$ and $\lambda>0$. Consider an absolute norm $\Omega$.
We have
\begin{equation}\label{eq:prox_nonneg}
\argmin_{\z \in \R{p}} \Big[ \frac{1}{2}\! \NormDeux{[\kappab]_+ \! - \z}^2 + \lambda  \Omega(\z)\Big]=
\argmin_{\z \in \R{p}_+} \Big[ \frac{1}{2}\! \NormDeux{\kappab \! - \z}^2 + \lambda  \Omega(\z)\Big].
\end{equation}
\end{lemma}
\begin{proof}
Let us denote by $\hat{\z}^+$ and $\hat{\z}$ the unique solutions of the left- and right-hand side of~(\ref{eq:prox_nonneg}) respectively.
Consider the normal cone $\N_{\R{p}_+}(\z_0)$ of $\R{p}_+$ at the point $\z_0$~\citep{Borwein2006}
and decompose $\kappab$ into its positive and negative parts, $\kappab=[\kappab]_+ + [\kappab]_-$. 
We can now write down the optimality conditions for the two convex problems above~\citep{Borwein2006}:
$\hat{\z}^+$ is optimal if and only if there exists $\w \in\partial\Omega(\hat{\z}^+)$ such that
$
\hat{\z}^+ - [\kappab]_+ +\lambda \w = \mathbf{0}.
$
Similarly, $\hat{\z}$ is optimal if and only if there exists 
$(\s,\u) \in \partial\Omega(\hat{\z}) \times \N_{\R{p}_+}(\hat{\z})$ such that
$
\hat{\z} - \kappab +\lambda \s + \u = \mathbf{0}.
$
We now prove that $[\kappab]_-=\kappab - [\kappab]_+$ belongs to $\N_{\R{p}_+}(\hat{\z}^+)$.
We proceed by contradiction. 
Let us assume that there exists $\z \in \R{p}_+$ such that $[\kappab]_-^\top (\z - \hat{\z}^+) > 0$.
This implies that there exists $j \in \{1,\dots,p\}$ for which
$[\kappab_j]_- < 0$ and $\z_j - \hat{\z}^+_j < 0$.
In other words, we have $0 \leq \z_j = \z_j-[\kappab_j]_+ < \hat{\z}^+_j = \hat{\z}^+_j - [\kappab_j]_+$.
With the assumption made on $\Omega$ and replacing $\hat{\z}^+_j$ by $\z_j$,
we have found a solution to the left-hand side of~(\ref{eq:prox_nonneg})
with a stricly smaller cost function than the one evaluated at $\hat{\z}^+$, hence the contradiction.
Putting the pieces together, we now have
$$
\hat{\z}^+ - [\kappab]_+ +\lambda \w  = \hat{\z}^+ - \kappab +\lambda \w + [\kappab]_-= \mathbf{0},\ 
\text{with}\ (\w,[\kappab]_-)\in  \partial\Omega(\hat{\z}^+) \times \N_{\R{p}_+}(\hat{\z}^+),
$$
which shows that $\hat{\z}^+$ is the solution of the right-hand side of~(\ref{eq:prox_nonneg}).
\end{proof}
\section{Counterexample for $\ell_q$-norms, with $q \notin \{1,2,\infty\}$.} \label{appendix:counter}
The result we have proved in Proposition~\ref{prop:one_pass_conv} in the specific setting where $\Norm{.}$ is the
$\ell_2$- or $\ell_\infty$-norm does not hold more generally for
$\ell_q$-norms, when $q$ is not in $\{1,2,\infty\}$. 
Let $q > 1$ satisfying this condition.
We denote by $q' \defin (1 - q^{-1})^{-1}$ the norm parameter dual to $q$.
We keep the same notation as in \refLemma{lem:proj_nested_gr} and
assume from now on that $\Norm{\u_{\gi}}_{q'}>t_g$ and $\Norm{\u_{\hi}}_{q'}>t_g+t_h$. 
These two inequalities guarantee that the vectors $\u_{\gi}$ and $\u_{\hi}-\xib^g$ do not lie in the interior of the $\ell_{q'}$-norm balls, of respective radius $t_g$ and $t_h$.

We show in this section that there exists a setting for which the conclusion of
\refLemma{lem:proj_nested_gr} does not hold anymore. We first focus on a necessary condition
of \refLemma{lem:proj_nested_gr}:
\UpperSpace
\begin{lemma}[Necessary condition of \refLemma{lem:proj_nested_gr}]\label{lem:necessary_cond_proj}~\\
Let $\Norm{.}$ be a $\ell_q$-norm, with $q \notin \{1,2,\infty\}$. If the conclusion of \refLemma{lem:proj_nested_gr} holds,
then the vectors $\xib_{\gi}^g$ and $\xib_{\gi}^h$ are linearly dependent.
\end{lemma}
\LowerSpace
\begin{proof}
According to our assumptions on $\u_{\gi}$ and $\u_{\hi}-\xib^g$, 
we have that $\Norm{\xib^g}_{q'}=t_g$ and $\Norm{\xib^h}_{q'}=t_h$.
In this case, we can apply the second optimality conditions of \refLemma{lem:opt_cond_proj}, 
which states that equality holds in the $\ell_q$-$\ell_{q'}$ H\"older inequality.
As a result, there exists $\rho_g,\rho_h>0$ such that for all $j$ in $g$:
\begin{equation}
 |\xib_j^g|^{q'} = \rho_g |\u_j - \xib_j^g|^q\quad \mathrm{and}\quad 
 |\xib_j^h|^{q'} = \rho_h |\u_j - \xib_j^g - \xib_j^h|^q.
\end{equation}
If the conclusion of \refLemma{lem:proj_nested_gr} holds---that is, we have 
$\xib^g = \Pi_{\|.\|_\ast \leq t_g}(\u_{\gi}-\xib_{\gi}^h)$,
notice that it is not possible to have the following scenarios, as proved below by contradiction:
\begin{itemize}
 \item If $\Norm{\u_{\gi}-\xib_{\gi}^h}_{q'} < t_g$, then we would have $\xib^g=\u_{\gi}-\xib_{\gi}^h$, 
 which is impossible since $\Norm{\xib^g}_{q'}=t_g$.
 \item If $\Norm{\u_{\gi}-\xib_{\gi}^h}_{q'} = t_g$, then we would have for all $j$ in $g$, 
 $|\xib_j^h|^{q'} = \rho_h |\u_j - \xib_j^g - \xib_j^h|^q = 0$, which implies that $\xib_{\gi}^h=0$ and $\Norm{\u_{\gi}}_{q'} = t_g$. 
 This is impossible since we assumed $\Norm{\u_{\gi}}_{q'} > t_g$.
\end{itemize}
We therefore have $\Norm{\u_{\gi}-\xib_{\gi}^h}_{q'} > t_g$ and using again the second optimality conditions of \refLemma{lem:opt_cond_proj}, 
there exists $\rho>0$ such that for all $j$ in $g$, $|\xib_j^g|^{q'} = \rho |\u_j - \xib_j^g - \xib_j^h|^q$.
Combined with the previous relation on $\xib_{\gi}^h$, we obtain for all $j$ in $g$,
$
|\xib_j^g|^{q'} = \frac{\rho}{\rho_h} |\xib_j^h|^{q'}.
$
Since we can assume without loss of generality that $\u$ only has nonnegative entries (see Lemma~\ref{lem:proj_same_sign}), 
the vectors $\xib^g$ and $\xib^h$ can also be assumed to have nonnegative entries, 
hence the desired conclusion.
\end{proof}
We need another intuitive property of the projection $\Pi_{\|.\|_\ast \leq t}$ to derive our counterexample:
\UpperSpace
\begin{lemma}[Order-preservation by projection]\label{lem:prop_proj}~\\
Let $\Norm{.}$ be a $\ell_q$-norm, with $q \notin \{1,\infty\}$ and $q'\defin1/(1-q^{-1})$. Let us consider the vectors $\kappab,\w \in \R{p}$ such that 
$
\kappab=\Pi_{\|.\|_\ast \leq t}(\w)=\argmin_{\Norm{\y}_{q'} \leq t}\|\y-\w\|_2,
$
with the radius $t$ satisfying $\Norm{\w}_{q'} > t$. If we have
$\w_i < \w_j$ for some $(i,j)$ in $\{1,\dots,p\}^2$, then it also holds that $\kappab_i < \kappab_j$.
\end{lemma}
\LowerSpace
\begin{proof}
Let us first notice that given the assumption on $t$, we have $\Norm{\kappab}_{q'} = t$.
The Lagrangian~$\mathcal{L}$ associated with the convex minimization problem underlying the definition of $\Pi_{\|.\|_\ast \leq t}$ can be written as
$$
	\mathcal{L}(\y,\alpha) = \frac{1}{2}\|\y-\w\|^2_2 + \alpha \big[ \Norm{\y}^{q'}_{q'} - t^{q'}\big],\ \text{with the Lagrangian parameter}\ \alpha \geq 0.
$$
At optimality, the stationarity condition for $\kappab$ leads to
$$
	\forall\ j\in\{1,\dots,p\},\ \kappab_j - \w_j + \alpha {q'} |\kappab_j|^{q'-1} = 0.
$$
We can assume without loss of generality that $\w$ only has nonnegative entries (see Lemma~\ref{lem:proj_same_sign}).
Since the components of $\kappab$ and $\w$ have the same signs (see Lemma~\ref{lem:proj_same_sign}), 
we therefore have $|\kappab_j|=\kappab_j \geq 0$, for all $j$ in $\{1,\dots,p\}$. Note that $\alpha$ cannot be equal to zero because of 
$\Norm{\kappab}_{q'} = t < \Norm{\w}_{q'}$.

Let us consider the continuously differentiable function $\varphi_w:  \kappa \mapsto \kappa - w + \alpha q' \kappa^{q'-1}$ defined on $(0,\infty)$.
Since $\varphi_w(0)=-w < 0$, $\lim_{\kappa\to\infty}\varphi_w(\kappa)=\infty$ and $\varphi_w$ is strictly nondecreasing, there exists a unique $\kappa^\ast_w > 0$ such that $\varphi_w(\kappa^\ast_w)=0$.
If we now take $w<v$, we have
$$
\varphi_v(\kappa^\ast_w)=\varphi_w(\kappa^\ast_w)+w-v=w-v<0=\varphi_v(\kappa^\ast_v).
$$
With $\varphi_v$ being strictly nondecreasing, we thus obtain $\kappa^\ast_w < \kappa^\ast_v$.
The desired conclusion stems from the application of the previous result to the stationarity condition of $\kappab$.
\end{proof}

Based on the two previous lemmas, we are now in position to present our counterexample:
\UpperSpace
\begin{proposition}[Counterexample]\label{prop:counterexample}~\\
Let $\Norm{.}$ be a $\ell_q$-norm, with $q \notin \{1,2,\infty\}$ and $q'\defin1/(1-q^{-1})$.
Let us consider $\G=\{g,h\}$, with $g \subseteq h \subseteq \{1,\dots, p\}$ and $|g|>1$.
Let $\u$ be a vector in $\R{p}$ that has at least two different nonzero entries in $g$, i.e., there exists $(i,j)$ in $g\times g$ such that $0<|\u_i|<|\u_j|$. 
Let us consider the successive projections
$$
\xib^g \defin \Pi_{\|.\|_\ast \leq t_g}(\u_{\gi}) 
	~\text{ and }~
        \xib^h \defin \Pi_{\|.\|_\ast \leq t_h}(\u_{\hi}-\xib^g) 
$$
with $t_g, t_h > 0$ satisfying $\Norm{\u_{\gi}}_{q'}>t_g$ and $\Norm{\u_{\hi}}_{q'}>t_g+t_h$. 
Then, the conclusion of Lemma~\ref{lem:proj_nested_gr}
does not hold.
\end{proposition}
\LowerSpace
\begin{proof}
We apply the same rationale as in the proof of \refLemma{lem:prop_proj}.
Writing the stationarity conditions for $\xib^g$ and $\xib^h$, we have for all $j$ in $g$
\begin{equation}
 \xib_j^g + \alpha q' (\xib_j^g)^{q'-1} - \u_j = 0,\quad \mathrm{and}\quad 
 \xib_j^h + \beta q' (\xib_j^h)^{q'-1}  - (\u_j - \xib_j^g) = 0, 
\end{equation}
with Lagrangian parameters $\alpha,\beta>0$.
We now proceed by contradiction and assume that 
$
\xib^g = \Pi_{\|.\|_\ast \leq t_g}(\u_{\gi}-\xib_{\gi}^h).
$
According to \refLemma{lem:necessary_cond_proj}, there exists $\rho>0$ such that for all $j$ in $g$, 
$
\xib_j^h=\rho \xib_j^g.
$
If we combine the previous relations on $\xib^g$ and $\xib^h$, we obtain for all $j$ in $g$, 
$$
 \xib_j^g = C (\xib_j^g)^{q'-1},\ \text{with}\ C\defin \frac{q'(\alpha-\beta \rho^{q'-1})}{\rho}.
$$
If $C < 0$, then we have a contradiction, since the entries of $\xib^g$ and $\u_{\gi}$ have the same signs.
Similarly, the case $C = 0$ leads a contradiction,  since we would have $\u_{\gi}=0$ and $\Norm{\u_{\gi}}_{q'}>t_g$.
As a consequence, it follows that $C > 0$ and for all $j$ in $g$, 
$
\xib_j^g = \exp\big\{\frac{\log(C)}{2-q'}\big\}, 
$
which means that all the entries of the vector $\xib_g^g$ are identical.
Using \refLemma{lem:prop_proj}, since there exists $(i,j) \in g\times g$ such that $\u_i < \u_j$, 
we also have $\xib_i^g < \xib_j^g$, which leads to a contradiction.
\end{proof}

\bibliography{main_bibliography}

\end{document}